\documentclass[runningheads]{llncs}
\pdfoutput=1
\usepackage[T1]{fontenc}
\usepackage{graphicx}
\usepackage{amssymb}
\usepackage{algorithm}
\usepackage{algorithmic}
\usepackage{amsmath}
\usepackage{bm}
\usepackage{graphicx}
\usepackage{soul}
\usepackage{subfigure}
\usepackage{todonotes}
\newcommand{\oomit}[1]{}

\begin{document}
\title{Safety Verification for Neural Networks Based on Set-boundary Analysis}

\author{Zhen Liang\inst{1} \and
Dejin Ren\inst{2} \and
Wanwei Liu\inst{1} \and
Ji Wang\inst{1} \and \\ 
Wenjing Yang\inst{1} \and 
Bai Xue \inst{2} \thanks{Corresponding author} 
}

\authorrunning{Z. Liang et al.}
\institute{College of Computer Science and Technology, National University of Defense Technology, Changsha, China \\
\email{\{liangzhen, wwliu, Jiwang, wenjing.yang\}@nudt.edu.cn}\\
 \and
Institute of Software CAS, Beijing, China\\
\email{\{rendj, xuebai\}@ios.ac.cn}}

%\author{Anonymous}
%\institute{Anonymous}
%
\maketitle              % typeset the header of the contribution
\begin{abstract}
Neural networks (NNs) are increasingly applied in safety-critical systems such as autonomous vehicles. However, they are fragile and are often ill-behaved. Consequently, their behaviors should undergo rigorous guarantees before deployment in practice. In this paper we propose a set-boundary reachability method to investigate the safety verification problem of NNs from a topological perspective. Given an NN with an input set and a safe set, the safety verification problem is to determine whether all outputs of the NN resulting from the input set fall within the safe set. In our method, the homeomorphism property of NNs is mainly exploited, which establishes a relationship mapping boundaries to boundaries. The exploitation of this property facilitates reachability computations via extracting subsets of the input set rather than the entire input set, thus controlling the wrapping effect in reachability analysis and facilitating the reduction of computation burdens for safety verification. The  homeomorphism property exists in some widely used NNs such as invertible NNs. Notable representations are invertible residual networks (i-ResNets) and Neural ordinary differential equations (Neural ODEs). For these NNs, our set-boundary reachability method only needs to perform reachability analysis on the boundary of the input set. For NNs which do not feature this property with respect to the input set, we explore subsets of the input set for establishing the local homeomorphism property, and then abandon these subsets for reachability computations. Finally, some examples demonstrate the performance of the proposed method.

\end{abstract}
\section{Introduction}
Machine learning has seen rapid growth due to the high amount of data produced in many industries and the increase in computation power. NNs have emerged as a leading candidate computation model for machine learning, which promote the prosperity of artificial intelligence in various fields, such as computer vision \cite{tian2021image,dahnert2021panoptic}, natural language processing \cite{yuan2021bartscore,karch2021grounding} and so on. In recent years, NNs are increasingly applied in safety critical systems. For example, a neural network has been implemented in the ACAS Xu airborne collision avoidance system for unmanned aircraft, which is a highly safety-critical system and currently
being developed by the Federal Aviation Administration. Consequently, to gain users’ trust and ease their concerns, it is of vital importance to ensure that NNs are able to produce safe outputs and satisfy the essential safety requirements before the deployment. 

Safety verification of NNs, which determines whether all outputs of an NN satisfy specified safety requirements via computing output reachable sets, has attracted a huge attention from different communities such as machine learning \cite{lomuscio2017approach,akintunde2019verification}, formal methods \cite{huang2017safety,tran2019star,liuww2020article}, and security \cite{wang2018efficient,gehr2018ai2}.  Because NNs are generally large, nonlinear, and non-convex, exact computation of output reachable sets is challenging. Although there are some methods on exact reachability analysis such as  SMT-based \cite{katz2017reluplex} and polyhedron-based approaches \cite{xiang2017reachable,tran2019parallelizable}, they are usually time-consuming and do not scale well. Moreover, these methods are limited to NNs with ReLU activation functions. Consequently, over-approximate reachability analysis, which mainly involves the computation of super sets of output reachable sets, is often resorted to in practice. The over-approximate analysis is usually more efficient and can be applied to more general NNs beyond ReLU ones. Due to these advantages, an increasing attention has been attracted and thus a large amount of computational techniques have been developed for over-approximate reachability analysis \cite{liu2021algorithms}.%e.g., mixed-integer linear program \cite{dutta2017output}, interval arithmetic- \cite{wang2018efficient}, zonotope- \cite{singh2018fast}, star- \cite{tran2019star}, and abstract-domain-\cite{singh2019abstract,yang2021improving} based approaches. 

Overly conservative over-approximations, however, often render many
safety properties unverifiable in practice. This conservatism
mainly results from the wrapping effect, which is the accumulation of over-approximation errors through layer-by-layer propagation. As the extent of the wrapping effect correlates strongly with the size of the input set \cite{xiang2018output}, techniques that partition the input set and independently compute output reachable sets of the resulting subsets are often adapted to reduce the wrapping effect, especially for large input sets. Such partitioning may, however, produce a 
 large number of subsets, which is generally exponential in the dimensionality. This will induce extensive demand on computation time and memory, often rendering existing reachability analysis techniques not suitable for safety verification of complex NNs in real applications. Therefore, exploring subsets of the input set rather than the entire input set could help reduce computation burdens and thus accelerate computations tremendously. %This is the objective of this work, which explores means of addressing the safety verification problem of NNs based on reachability analysis of just a small subset of the initial input set.%namely a set enclosing its boundary. 

%Neural networks (NN, for short) have emerged as a leading candidate computation model for deep learning (DL) in recent years, which promote the prosperity of artificial intelligence (AI) in various fields, such as computer vision, natural language processing, speech recognition and so on. 

In this work we investigate the safety verification problem of NNs from the topological perspective and extend the set-boundary reachability method, which is originally proposed for verifying safety properties of systems modeled by ODEs in \cite{xue2016reach}, to safety verification of NNs. In \cite{xue2016reach}, the set-boundary reachability method only performs over-approximate reachability analysis on the initial set's boundary rather than the entire initial set to address safety verification problems. It was built upon the homeomorphism property of ODEs. This nice property also widely exists in NNs, and typical NNs are invertible NNs such as neural ODEs \cite{chen2018neural} and invertible residual networks \cite{behrmann2019invertible}. Consequently, it is straightforward to extend the set-boundary reachability method to safety verification of these NNs, just using the boundary of the input set for reachability analysis which does not involve reachability computations of interior points and thus reducing computation burdens in safety verification. Furthermore, we extend the set-boundary reachability method to general NNs via exploiting the local homeomorphism property with respect to  the input set. This exploitation is instrumental for constructing a subset of the input set for reachability computations, which is gained via removing a set of points in the input set such that the NN is a homeomorphism with respect to them. The above methods of extracting subsets for performing reachability computations can also be applied to intermediate layers of NNs rather than just between the input and output layers. Finally, we demonstrate the performance of the proposed method on several examples.

Main contributions of this paper are listed as follows.

\begin{itemize}
    \item We investigate the safety verification problem of NNs from the topological perspective. More concretely, we exploit the homeomorphism property and aim at extracting a subset of the input set rather than the entire input set for reachability computations. To the best of our knowledge, this is the first work on the use of the homeomorphism property to address safety verification problems of NNs. This might on its own open research directions on digging into topological properties of facilitating reachability computations for NNs. 
    \item  The proposed method is able to enhance the capabilities and performances of existing reachability computation methods for safety verification of NNs via reducing computation burdens. Based on the homeomorphism property, the computation burdens of solving the safety verification problem can be reduced for invertible NNs. We further show that the computation burdens can also be reduced for more general NNs via exploiting this property on subsets of the input set. 
%    \item Moreover, we propose a general framework of reachable set computation (RSCBA) integrated with existing verification tools, not limited to specific abstract domains. Experiments on various neural networks illustrate  the performance on the tightness improvement of reachable set computation of neural networks.
\end{itemize}

%The remainder of this paper is organized as follows. First, we provide an overview of the most relevant research. Afterwards, we formulate the safety verification problem of interest in this work and then elucidate our set-boundary reachability method for addressing the safety verification problem of invertible NNs and more general NNs. Following this, we demonstrate our method on several interesting examples. Finally, we summarize this paper and give possible research directions in the future.

\section{Related Work}
%\subsection{Reachability Verification of NNs.}
There has been a dozen of works on safety verification of NNs. The first work on DNN verification was published in \cite{pulina2010abstraction}, which focuses on DNNs with Sigmoid activation functions via a partition-refinement approach. Later, Katz
et al. \cite{katz2017reluplex} and \cite{ehlers2017formal} independently implemented Reluplex and Planet, two SMT solvers to verify DNNs with ReLU activation function on properties expressible with SMT constraints. 

%\cite{katz2017reluplex} proposed the Reluplex algorithm, which is sound and complete to verify fully-connected NNs with ReLU activation functions, and then the Reluplex was incorporated into the Marabou framework \cite{katz2019marabou}, which removes the limitation of ReLU function.  \cite{ehlers2017formal} presented a similar method, which however, utilizes linear approximation to over-approximate the behaviours of NNs. \cite{lomuscio2017approach}  encoded the readability analysis of fully-connected NNs into mixed integer linear programs. However, these methods are not efficient. Consequently,  \cite{cheng2017maximum,dutta2017output,bunel2018piecewise} proposed heuristics, Sherlock algorithm and brand and bound to accelerate the computations, respectively.

Recently, methods based on abstract interpretation attracts more attention, which is to propagate sets in a sound (i.e., over-approximate) way \cite{cousot1977abstract} and is more efficient. There are many widely used abstract domains, such as intervals \cite{wang2018efficient}, and star-sets \cite{tran2019star}. A method based on zonotope abstract domains is proposed in \cite{gehr2018ai2}, which works for any piece linear activation function with great scalability. Then, it is further improved \cite{singh2018fast} for obtaining tighter results via imposing abstract transformation on  ReLU, Tanh and Sigmoid. \cite{singh2018fast} proposed specialized abstract zonotope transformers for handling NNs with ReLU, Sigmoid and Tanh activation functions. \cite{singh2019abstract} proposes an abstract domain that combines floating point polyhedra with intervals to over-approximate output reachable sets. Subsequently, a spurious region guided approach is proposed to infer tighter output reachable sets \cite{yang2021improving} based on the method in \cite{singh2019abstract}. \cite{dutta2017output} abstracts an NN by a polynomial, which has the advantage that dependencies can in principle be preserved. This approach can be precise in practice for small input sets. Afterwards, \cite{huang2019reachnn} approximates Lipschitz-continuous NNs with Bernstein polynomials.  \cite{ivanov2020verifying} transforms a neural network with Sigmoid activation functions into a hybrid automaton and then uses existing reachability analysis methods for the hybrid automaton to perform reachability computations. \cite{xiang2018output} proposed a maximum sensitivity based approach for  solving safety verification problems for multi-layer perceptrons with monotonic activation functions. In this approach, an exhaustive search of the input set is enabled by  discretizing input space to compute output reachable set which consists of a union of reachtubes.

%CROWN-IBP \cite{zhang2019towards} combines the tight linear relaxation on top of verification bound (CROWN) \cite{gowal2018effectiveness} and interval bound propagation (IBP) method \cite{zhang2018efficient} to certify the existence of adversarial examples. ExactReach \cite{xiang2017reachable} computes the reachable set via retaining a set of polytopes on the ReLU activated NNs. Whereas, the number of the polytopes would grow exponentially, so it does not scale even though it computes an exact reachability result.

%It trades efficiency with accuracy. 

%is a balance between the effectiveness and efficiency, that is to say, the finer the partition, the tighter the reachable set, the slower the computation.

% Due to the accuracy loss caused by the over approximation of the real output set, it is not possible to ensure that all the properties satisfied in NNs can be verified, so it is necessary to consider the refinement of the computation of the reachable set. The existing computation methods takes the refinement from the view of an entire set or its partitioned subsets, whereas, 
 
Neural ODEs were first introduced in 2018, which exhibit considerable computational efficiency on time-series modeling tasks \cite{chen2018neural}. Recent years have witnessed an increase use of them on real-world applications \cite{lechner2020neural,hasani2020natural}. However, the verification techniques for Neural ODEs are rare and still in fancy. The first reachability technique for Neural ODEs 
 appeared in \cite{grunbacher2021verification}, which proposed Stochastic Lagrangian reachability, an abstraction-based technique for constructing an over-approximation of the output reachable set with probabilistic guarantees. Later, this method was improved and implemented in a tool GoTube \cite{gruenbacher2021gotube}, which is able to perform reachability analysis for long time horizons. Since these methods only provide stochastic bounds on the computed over-approximation and thus cannot provide formal guarantees on the satisfaction of safety properties, \cite{lopez2022reachability} presented a deterministic verification framework for a general class of Neural ODEs with multiple continuous- and discrete-time layers. 

 Based on entire input sets, all the aforementioned works focus on developing computational techniques for reachability analysis and safety verification of appropriate NNs. In contrast, the present work shifts this focus to topological analysis of NNs and guides reachability computations on subsets of the input set rather than the entire input set, reducing computation burdens and thus increasing the power of existing safety verification methods for NNs. Although there are studies on topological properties of NNs \cite{behrmann2019invertible,dupont2019augmented,naitzat2020topology}, there is no work on the utilization of homeomorphism property to analyze their reachability and safety verification problems, to the best of our knowledge.

%\subsection{Homeomorphism Property in NNs}
%The main research on homeomorphism in NNs can be categorized into two classes, on invertible residual networks (i-ResNet) \cite{behrmann2019invertible} and neural ordinary differential equations (neural ODEs) \cite{chen2018neural}. The former is usually utilized in flow models to reconstruct the data under another distribution \cite{jacobsen2018revnet, behrmann2019invertible}, or to design lightweight neural network models, in terms of storage and training \cite{gomez2017reversible}. As for the latter, homeomorphism is mainly brought to characterize the expression ability and improvement of neural ODEs \cite{zhang2020invariance}. In this paper, our purpose of working on the homeomorphism property is to provide an rigorous guarantee for boundary analysis. 

\section{Preliminaries}
In this section, we give an introduction on the safety verification problem of interest for NNs and homeomorphisms. Throughout this paper, given a set $\Delta$, $\Delta^{\circ}$, $\partial \Delta$ and $\overline{\Delta}$ respectively denotes its interior, boundary and the closure. 

%the problem formulation of computing reachable sets and a high-level solution framework from the perspective of boundary analysis then. We generalize the boundary analysis to the cases of invertible residual networks (i-ResNet), neural ordinary differential equations (neural ODEs) and general NNs, afterwards.
NNs, also known as artificial NNs, are a subset of machine learning and are at the heart of deep learning algorithms. It works by using interconnected nodes or neurons in a layered structure that resembles a human brain, and is generally composed of three layers: an input layer, hidden layers and an output layer. Mathematically, it is a mathematical function $\bm{N}(\cdot): \mathbb{R}^n\rightarrow \mathbb{R}^m$, where $n$ and $m$ respectively denote the dimension of the input and output of the NN.

\subsection{Problem Statement}
Given an input set $\mathcal{X}_{in}$, the output reachable set of an NN $\bm{N}(\cdot): \mathbb{R}^n \rightarrow \mathbb{R}^m$ is stated by the following definition.

\begin{definition}
\label{safety}
For a given neural network $\bm{N}(\cdot): \mathbb{R}^n \rightarrow \mathbb{R}^m$, with an input set $\mathcal{X}_{in} \subseteq \mathbb{R}^n$, the output reachable  set $\mathcal{R}(\mathcal{X}_{in})$ is defined as
\[\mathcal{R}(\mathcal{X}_{in})=\{\bm{y}\in \mathbb{R}^m \mid \bm{y}=\bm{N}(\bm{x}), \ \bm{x}\in \mathcal{X}_{in}\}.\]
\end{definition}

The safety verification problem is formulated in Definition \ref{safety1}.
\begin{definition}[Safety Verification Problem]
\label{safety1}
Given a neural network $\bm{N}(\cdot): \mathbb{R}^n \rightarrow \mathbb{R}^m$, an input set $\mathcal{X}_{in}\subseteq \mathbb{R}^n$ which is compact, and a safe set $\mathcal{X}_s\subseteq \mathbb{R}^m$ which is simply connected, the safety verification problem is to verify that 
 \[\forall \bm{x}_0\in \mathcal{X}_{in}. \ \bm{N}(\bm{x}_0) \in \mathcal{X}_s.\]
\end{definition}

In topology, a simply connected set is a path-connected set where one can continuously shrink any simple closed curve into a point while remaining in it. The requirement that the safe set $\mathcal{X}_{s}$ is a simply  connected set is not strict, since many widely used sets such as intervals, ellipsoids, convex polyhedra and zonotopes are simply connected. 

Obviously, the safety property that $\forall \bm{x}_0\in \mathcal{X}_{in}. \  \bm{N}(\bm{x}_0) \in \mathcal{X}_s$ holds if and only if $\mathcal{R}(\mathcal{X}_{in})\subseteq \mathcal{X}_s$. However, it is challenging to compute the exact output reachable set $\mathcal{R}(\mathcal{X}_{in})$ and thus  an over-approximation $\Omega(\mathcal{X}_{in})$, which is a super set of the set $\mathcal{R}(\mathcal{X}_{in})$ (i.e., $\mathcal{R}(\mathcal{X}_{in})\subseteq \Omega(\mathcal{X}_{in})$), is commonly resorted to in existing literature for formally reasoning about the safety property. If $\Omega(\mathcal{X}_{in}) \subseteq \mathcal{X}_s$, the safety property that $\forall \bm{x}_0\in \mathcal{X}_{in}.\  \bm{N}(\bm{x}_0) \in \mathcal{X}_s$ holds.

%Different from existing methods of using entire sets (input set or intermediate set) to compute an over-approximation of the output set for safety verification, we in this work carefully explore boundaries of sets to perform over-approximate reachability analysis of subsets of the input set for safety verification. 

\subsection{Homeomorphisms}
%As studied in \cite{xue2016reach}, if the given neural network $\bm{N}(\cdot): \mathcal{X}_{in} \rightarrow \mathcal{R}(\mathcal{X}_{in})$ is a homeomorphism, the boundary $\partial \mathcal{R}(\mathcal{X}_{in})$ of the output reachable set can be obtained via propogating the input set's boundary. 
In this subsection, we will recall the definition of a homeomorphism, which is a map between spaces that preserves all topological properties.
\begin{definition}
A map $h: \mathcal{X}\rightarrow \mathcal{Y}$ with $\mathcal{X},\mathcal{Y}\subseteq \mathbb{R}^n$ is a homeomorphism with respect to $\mathcal{X}$ if it is a continuous bijection and its inverse $h^{-1}(\cdot): \mathcal{Y}\rightarrow \mathcal{X}$ is also continuous. 
\end{definition}

Homeomorphisms are continuous functions that preserve topological properties, which map boundaries to boundaries and interiors to interiors \cite{massey2019basic}. 
\begin{proposition}
 Suppose sets $\mathcal{X},\mathcal{Y}\subseteq \mathbb{R}^n$ are compact. If a map $h(\cdot): \mathcal{X}\rightarrow \mathcal{Y}$ is a homeomorphism, then  $h$ maps the boundary of the set $\mathcal{X}$ onto the boundary of the set $\mathcal{Y}$, and the interior of  the set $\mathcal{X}$ onto the interior of the set $\mathcal{Y}$.
\end{proposition}

\begin{figure}[htbp]
\centering
\subfigure[A homeomorphic map]{
\begin{minipage}[t]{0.5\linewidth}
\centering
\includegraphics[width=1.5in]{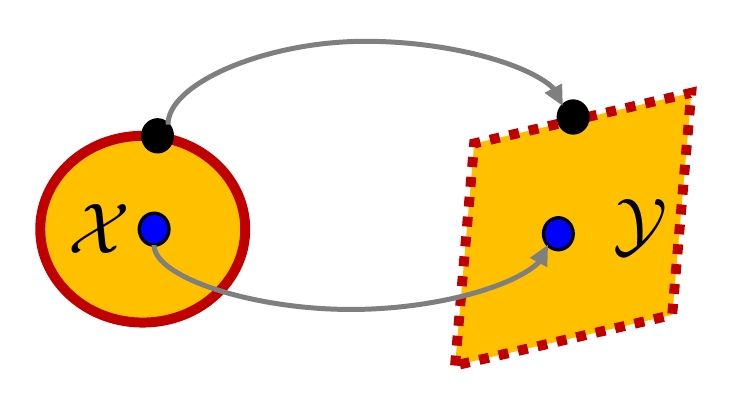}
%\caption{fig1}
\label{home}
\end{minipage}%
}%
\subfigure[A non-homeomorphic map]{
\begin{minipage}[t]{0.5\linewidth}
\centering
\includegraphics[width=1.5in]{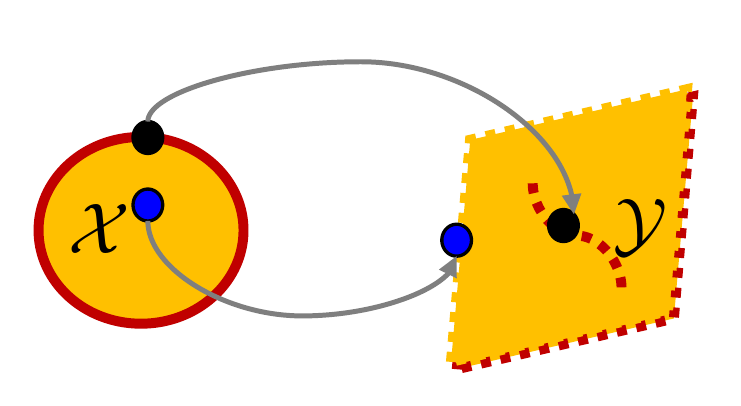}
%\caption{fig2}
\label{non-homeo}
\end{minipage}%
}%
\centering
\caption{Homeomorphic and non-homeomorphic maps}
\label{ill-home}
\end{figure}

Based on this property, \cite{xue2016reach} proposed a set-boundary reachability method for safety verification of ODEs, via only propagating the initial set's boundary. Later, this method was extended to a class of delay differential equations \cite{xue2020over}.

%Different from existing methods of performing reachability

%We term the methods with this problem encoding as {\textit{Set Propagation Methods}} herein. Obviously, $\cup_{x\in \mathcal{X}_{in}} g(x)$ is the real and optimal set, called $\mathcal{X}_{out}$, however, it is extremely intractable due to the nonlinearity and nonconvex appeared in  NNs. Therefore, different abstract domain is utilized to over approximate (contain) the optimal set as exact as possible, such as interval, polytope, zonotope, polynomial zonotope and so on.

%\subsection{Computation Framework via Boundary Analysis.}

\section{Safety Verification Based on Boundary Analysis}
In this section we introduce our set-boundary reachability method for addressing the safety verification problem in the sense of Definition \ref{safety}. We first consider invertible NNs in Subsection \ref{Sec-inn}, and then extend the method to more general NNs in Subsection \ref{Sec-noninn}.

\subsection{Safety Verification on Invertible NNs}\label{Sec-inn}
In this subsection we introduce  our set-boundary reachability method for safety verification on invertible NNs, which relies on the homeomorphism property of these NNs.

Invertible NNs, such as i-RevNets \cite{jacobsen2018revnet}, RevNets \cite{gomez2017reversible}, i-ResNets \cite{behrmann2019invertible} and Neural ODEs \cite{chen2018neural}, are NNs with invertibility by designed architectures, which can reconstruct inputs from their outputs. These NNs are continuous bijective maps. Based on the facts that $\mathcal{X}_{in}$ is compact, they are homeomorphisms [Corollary 2.4, \cite{joshi1983introduction}]\footnote{A continuous bijection from a compact space onto a Hausdorff space is a homeomorphism. (Euclidean space and any subset of Euclidean space is Hausdorff.)}. In existing literature, many invertible NNs are constructed by requiring their Jacobian determinants to be 
non-zero \cite{ardizzone2018analyzing}. Consequently, based on the inverse function theorem \cite{krantz2002implicit}, these NNs are homeomorphisms. In the present work, we also use Jacobian determinants to justify the invertibility of some NNs. It is noteworthy that Jacobian determinants being non-zero is a sufficient but not necessary condition for homeomorphisms and the reason why we resort to this requirement lies in the simple and efficient computations of Jacobian determinants with interval arithmetic. However, this demands the differentiability of NNs. Thus, this  technique of computing Jacobian determinants to determining homeomorphisms is not applicable to NNs with ReLU activation functions.

%As for NNs with non-differential ReLU activation functions, one may find that the computation of Jacobian determinants excludes the applicability of our method, nevertheless, it is difficult to establish homeomorphisms in fact since the mappings on inactive states (i.e., $(-\infty, 0] \rightarrow 0$) is not injective.} 

Based on the homeomorphism property of mapping the input set's boundary onto the output reachable set's boundary, we propose a set-boundary reachability method for safety verification of invertible NNs, which just performs the over-approximate reachability analysis on the input set's boundary. Its computation procedure is presented in Algorithm \ref{alg: iNNs}. %Notably, invertible NNs referred to in this paper feature the homeomorphism property and vice versa.
\begin{algorithm}[tb]
\caption{Safety Verification Framework for Invertible NNs Based on Boundary Analysis}
\label{alg: iNNs}
\textbf{Input}: an invertible NN $\bm{N}(\cdot): \mathbb{R}^n \rightarrow \mathbb{R}^n$, an input set $\mathcal{X}_{in}$ and a safe set $\mathcal{X}_s$.\\
\textbf{Output}: \textbf{Safe} or \textbf{Unknown}.
\begin{algorithmic}[1] %[1] enables line numbers
\STATE extract the boundary $\partial \mathcal{X}_{in}$ of the input set $\mathcal{X}_{in}$;
\STATE apply existing methods to compute an over-approximation $\Omega(\partial \mathcal{X}_{in})$;
\IF {$\Omega(\partial \mathcal{X}_{in})\subseteq \mathcal{X}_s$} 
\STATE return \textbf{Safe}
%\STATE $\mathcal{B}_{out}=T(\mathbb{N}, \mathcal{B}_{in})$.\label{Line 2}
%\STATE $\widehat{\mathcal{X}_{out}}=$CVX$(\mathcal{B}_{out})$.\label{Line 3}
\ELSE
\STATE return \textbf{Unknown}
\ENDIF
%\STATE $\mathcal{X}_{in}=\widehat{\mathcal{X}_{out}}$.
%\STATE \textbf{return} $\widehat{\mathcal{X}_{out}}$
\end{algorithmic}
\end{algorithm}

\begin{remark}
   In the second step  of Algorithm \ref{alg: iNNs}, we may take partition operator on the input set's boundary to refine the computed over-approximation for addressing the safety verification problem.  
\end{remark}

\begin{theorem}[Soundness]
\label{sound}
If Algorithm \ref{alg: iNNs} returns \textbf{\em Safe},  the safety property in the sense of Definition \ref{safety} holds.    
\end{theorem}
\begin{proof}
It is equivalent to show that if $\mathcal{R}(\partial \mathcal{X}_{in})\subseteq \mathcal{X}_s$, 
\[\forall \bm{x}_0\in \mathcal{X}_{in}. \ \bm{N}(\bm{x}_0)\in \mathcal{X}_s.\]
The conclusion holds by Lemma 3 in \cite{xue2016reach}.
\end{proof}

%\textcolor{red}{The property that the boundary of the output reachable set is a subset of the output reachable set of the input set's boundary (i.e., $\partial \mathcal{R}(\mathcal{X}_{in})\subseteq \mathcal{R}(\partial \mathcal{X}_{in})$) exists in some NNs.  }

In order to enhance the understanding of Algorithm \ref{alg: iNNs} and its benefits, we use a sample example  to illustrate it. %An illustration of using Algorithm \ref{alg: iNNs} for safety verification of invertible NNs is presented in Example \ref{ex1}.
\begin{example}
\label{ex1}
Consider an {\em NN} from \cite{xiang2018output}, which has 2 inputs, 2 outputs and 1 hidden layer consisting of 5 neurons. The input set is $\mathcal{X}_{in}=[0,1]^2$. Its boundary is $\partial \mathcal{X}_{in}=\cup_{i=1}^4 \mathcal{B}_i$, where $\mathcal{B}_1=[0,0]\times [0,1]$, $\mathcal{B}_2=[1,1]\times [0,1]$, $\mathcal{B}_3=[0,1]\times [0,0]$ and $\mathcal{B}_4=[0,1]\times [1,1]$. The activation functions for the hidden layer and the output layer are $\mathtt{Tanh}$ and $\mathtt{Purelin}$ functions, respectively, whose weight matrices and bias vectors can be found in Example 1 in \cite{xiang2018output}. For this neural network, based on interval arithmetic, we can show that the determinant of the Jacobian matrix $\frac{\partial \bm{y}}{\partial \bm{x}_0}=\frac{\partial \bm{N}(\bm{x}_0)}{\partial \bm{x}_0}$ is non-zero for any $\bm{x}_0\in \mathcal{X}_{in}$. Therefore, this NN is invertible and the map $\bm{N}(\cdot):\mathcal{X}_{in} \rightarrow \mathcal{R}(\mathcal{X}_{in})$ is a homeomorphism with respect to the input set $\mathcal{X}_{in}$, leading to $\mathcal{R}(\partial \mathcal{X}_{in})=\partial \mathcal{R}(\mathcal{X}_{in})$. This statement is also verified via the visualized results in Fig. \ref{illu_eps1}. 

The homeomorphism property facilitates the reduction of the wrapping effect in over-approximate reachability analysis and thus reduces computation burdens in addressing the safety verification problem in the sense of Definition \ref{safety}. For this example, with the safe set $\mathcal{X}_s=[-3.85, -1.85]\times[-0.9,1.7]$, we first %do not take partition operator on%
take the input set and its boundary for reachability computations. Based on interval arithmetic, we respectively compute over-approximations $\Omega(\mathcal{X}_{in})$ and $\Omega(\partial \mathcal{X}_{in})$, which are illustrated in Fig. \ref{illu_eps10}. Although the approximation $\Omega(\partial \mathcal{X}_{in})$ is indeed smaller than $\Omega(\mathcal{X}_{in})$, it still renders the safety property unverifiable. We next take partition operator for more accurate reachability computations. If the entire input set is used, we can successfully verify the safety property when the entire input set is divided into  $10^4$ small intervals of equal size. In contrast, our set-boundary reachability method just needs $400$ equal partitions on the input set's boundary, significantly reducing the computation burdens. The reachability results, i.e., the computation of $\Omega(\partial \mathcal{X}_{in})$, are illustrated in Fig. \ref{illu_eps11}. 

\iffalse
\begin{figure}[htbp]
\center
\includegraphics[width=1.7in]{ illu_1.pdf} 
%and  each component of $\bm{c}$ is in $[-10^2,10^2]$ 
\caption{Output reachable sets estimated via Monte-Carlo Methods: blue region--$\mathcal{R}(\mathcal{X}_{in})$; red region--$\mathcal{R}(\partial \mathcal{X}_{in})$.}
\label{illu_eps1}
\end{figure}

\begin{figure}[htb!]
\center
\includegraphics[width=2.5in,height=2.0in]{ illu_10.pdf} 
%and  each component of $\bm{c}$ is in $[-10^2,10^2]$ 
\caption{blue line--$\partial \Omega(\mathcal{X}_{in})$; red line--$\Omega(\partial \mathcal{X}_{in})$; green line --$\partial \mathcal{X}_s$}
\label{illu_eps10}
\end{figure}

\begin{figure}[htb!]
\center
\includegraphics[width=2.5in,height=2.0in]{ illu_11.pdf} 
%and  each component of $\bm{c}$ is in $[-10^2,10^2]$ 
\caption{blue region--$\Omega(\mathcal{X}_{in})$; red region--$\Omega(\partial \mathcal{X}_{in})$; green line --$\partial \mathcal{X}_s$}
\label{illu_eps11}
\end{figure}
\fi
\end{example}

\begin{figure}[htbp]
\centering
\subfigure[\textcolor{blue}{$\mathcal{R}(\mathcal{X}_{in})$} and \textcolor{red}{$\mathcal{R}(\partial \mathcal{X}_{in})$} estimated via Monte-Carlo method]{
\begin{minipage}[t]{0.32\linewidth}
\centering
\includegraphics[width =1.7in, height=1.2in]{ 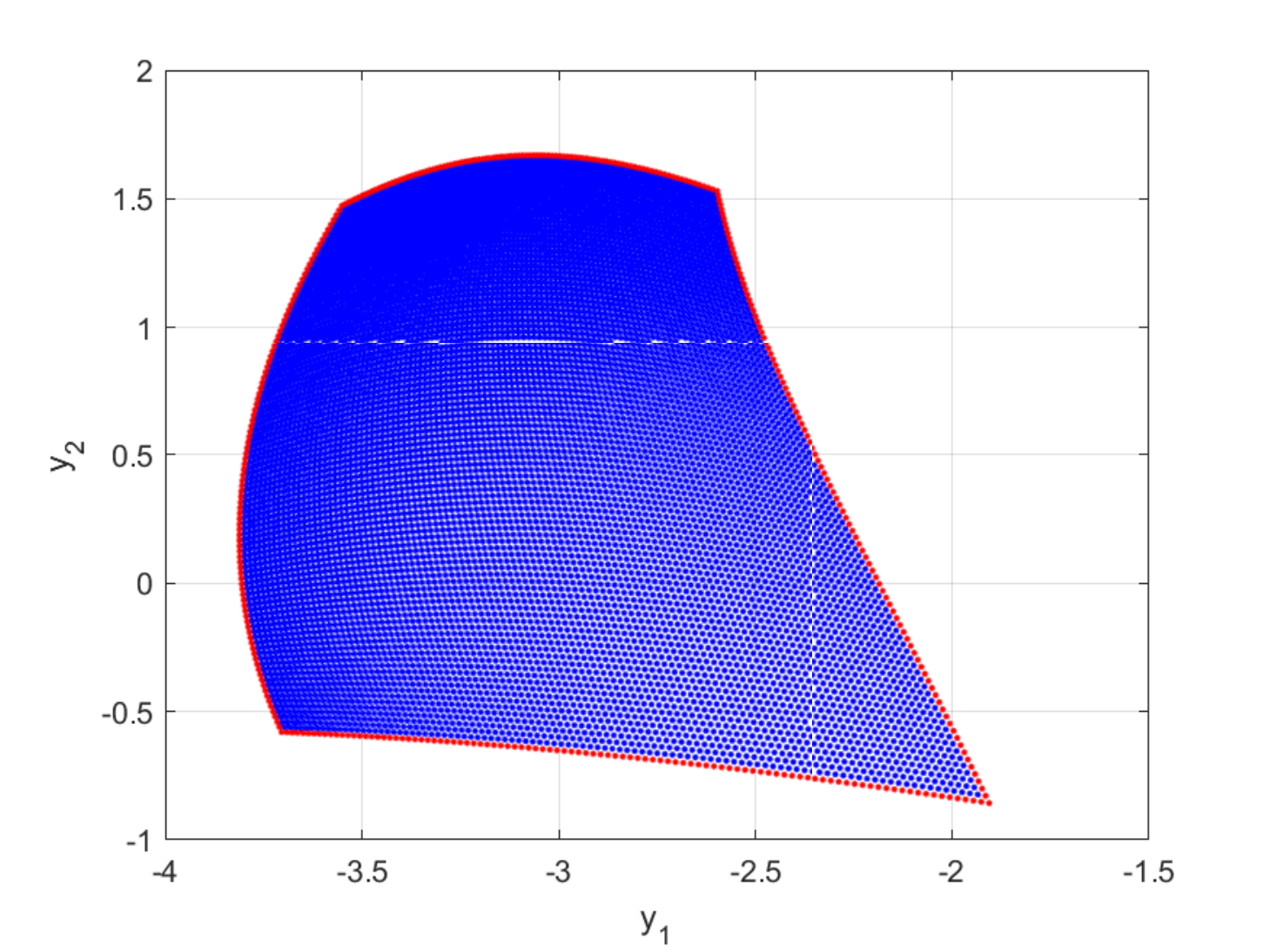}
%\caption{fig1}
\label{illu_eps1}
\end{minipage}%
}%
\subfigure[\textcolor{blue}{$\partial \Omega(\mathcal{X}_{in})$}; \textcolor{red}{$\Omega(\partial \mathcal{X}_{in})$}; \textcolor{green}{$\partial \mathcal{X}_s$}]{
\begin{minipage}[t]{0.32\linewidth}
\centering
\includegraphics[width =1.7in, height=1.2in]{ 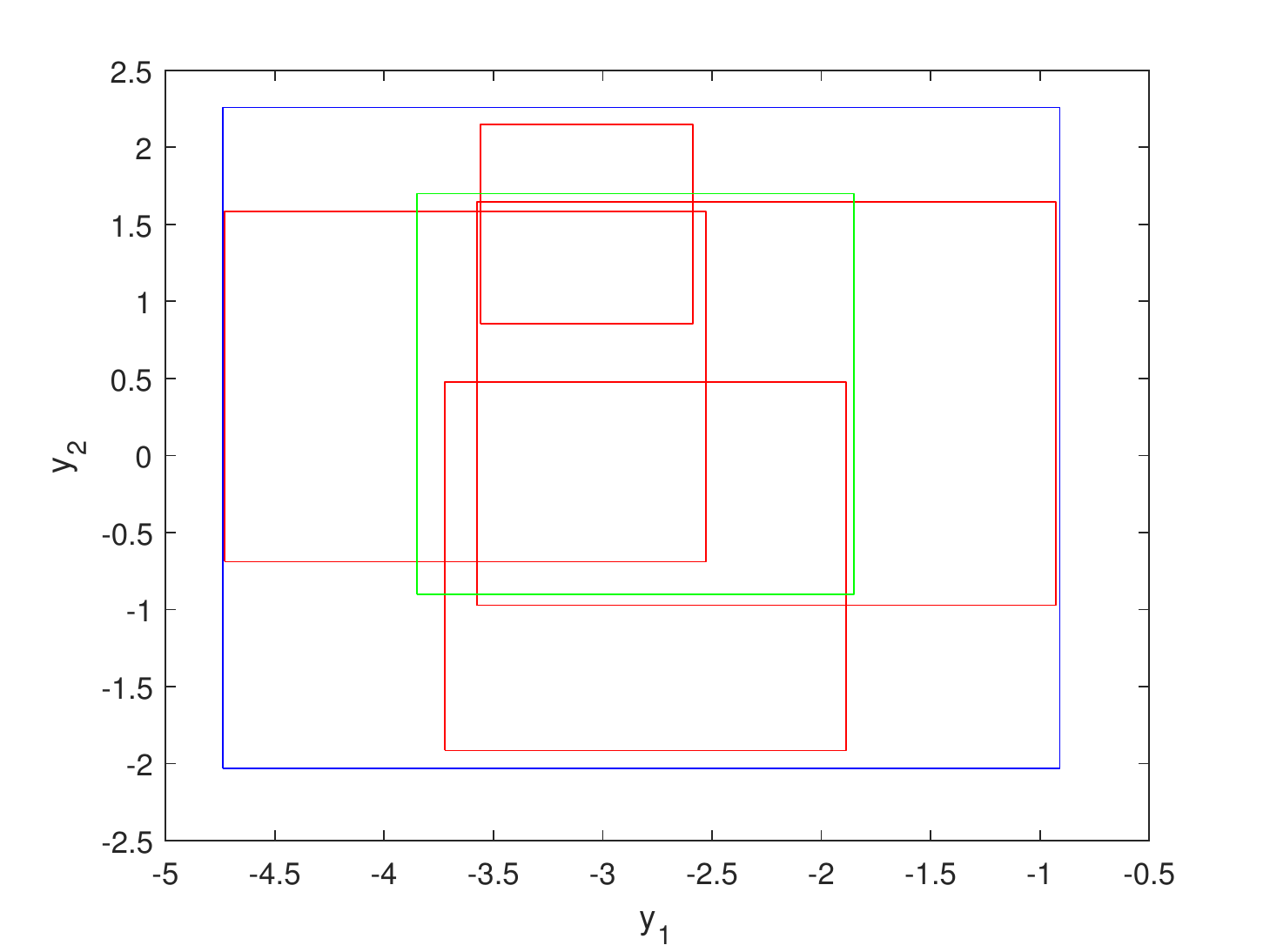}
%\caption{fig1}
\label{illu_eps10}
\end{minipage}%
}%
\subfigure[\textcolor{blue}{$\Omega(\mathcal{X}_{in})$}; \textcolor{red}{$\Omega(\partial \mathcal{X}_{in})$}; \textcolor{green}{$\partial \mathcal{X}_s$}]{
\begin{minipage}[t]{0.32\linewidth}
\centering
\includegraphics[width =1.7in, height=1.2in]{ 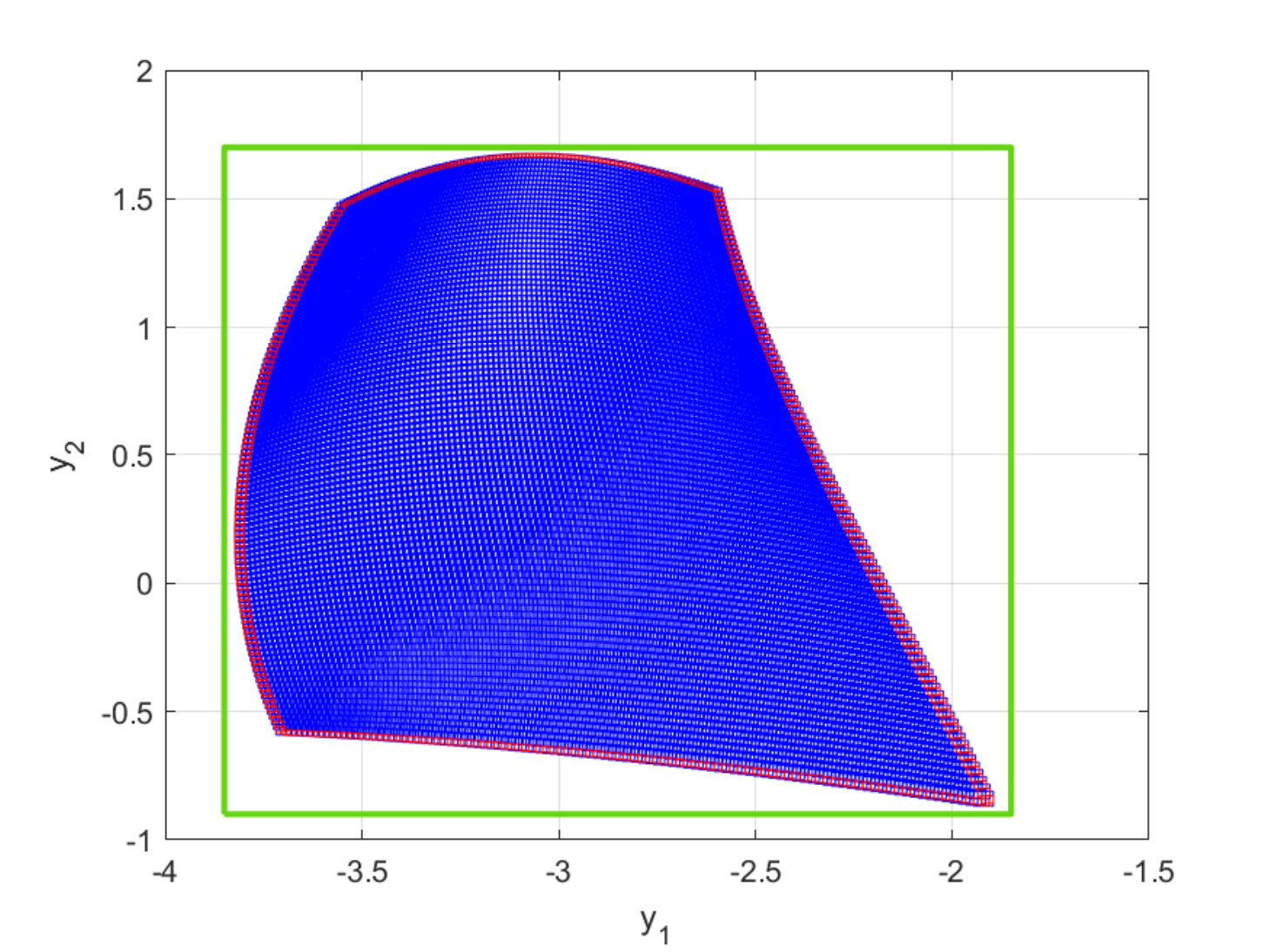}
%\caption{fig2}
\label{illu_eps11}
\end{minipage}%
}%
\centering
\caption{Illustrations on Example \ref{ex1}}
\label{ex1_figure}
\end{figure}

\subsection{Safety Verification on Non-invertible NNs}\label{Sec-noninn}

When an NN has the homeomorphism property, we can use Algorithm \ref{alg: iNNs} to address the safety verification problem in the sense of Definition \ref{safety}.  However, not all of NNs have such a nice property. In this subsection we extend the set-boundary reachability method to safety verification of non-invertible NNs, via analyzing the homeomorphism property of NNs with respect to subsets of the input set $\mathcal{X}_{in}$.
\begin{example}
\label{ex2}
Consider an NN from \cite{xiang2018reachability}, which has 2 inputs, 2 outputs and 1 hidden layer consisting of 7 neurons. The input set is $\mathcal{X}_{in}=[-1,1]^2$. The activation functions for the hidden layer and the output layer are $\mathtt{Tanh}$ and $\mathtt{Purelin}$ functions, respectively, whose  weight matrices and bias vectors can be found in Example 4.3 in \cite{xiang2018reachability}. For this neural network, the boundary of the output reachable set, i.e.,$\partial \mathcal{R}(\mathcal{X}_{in})$, is not included in the output reachable set of the input set's boundary $\mathcal{R}(\partial \mathcal{X}_{in})$. This statement is visualized in Fig. \ref{illu_eps2}.

\iffalse
\begin{figure}[htb!]
\center
\includegraphics[width=1.7in]{ illu_2.pdf} 
%and  each component of $\bm{c}$ is in $[-10^2,10^2]$ 
\caption{blue region-- $\mathcal{R}(\mathcal{X}_{in})$; red region--$\mathcal{R}(\partial \mathcal{X}_{in})$.}
\label{illu_eps2}
\end{figure}
\fi
\end{example}

Example \ref{ex2} presents us an NN, whose mapping does not admit the homeomorphism property with respect to the input set and the output reachable set. However, the NN may feature the homeomorphism property with respect to a subset of the input set. This is illustrated in Example \ref{ex3}.
\begin{example}
    \label{ex3}
    Consider the NN in Example \ref{ex2} again. We divide the input set $\mathcal{X}_{in}$ into $4\times 10^4$ small intervals of equal size and verify whether the NN is a homeomorphism with respect to each of them based on the use of interval arithmetic to determine the determinant of the  corresponding Jacobian matrix $\frac{\partial \bm{y}}{\partial \bm{x}_0}=\frac{\partial \bm{N}(\bm{x}_0)}{\partial \bm{x}_0}$. The blue region in Fig. \ref{illu_eps3} is the set of intervals, which features the NN with the homeomorphism property. The number of these intervals is 31473. For simplicity, we denote these intervals by $\mathcal{A}$. 

\iffalse
    \begin{figure}[htb!]
\center
\includegraphics[width=1.8in]{ illu_3.pdf} 
%and  each component of $\bm{c}$ is in $[-10^2,10^2]$ 
\caption{blue region-- the set $\mathcal{A}$}
\label{illu_eps3}
\end{figure}
\fi
\end{example}

\begin{figure}[htbp]
\centering
\subfigure[\textcolor{blue}{$\mathcal{R}(\mathcal{X}_{in})$}; \textcolor{red}{$\mathcal{R}(\partial \mathcal{X}_{in})$}]{
\begin{minipage}[t]{0.32\linewidth}
\centering
\includegraphics[width =1.7in, height=1.2in]{ 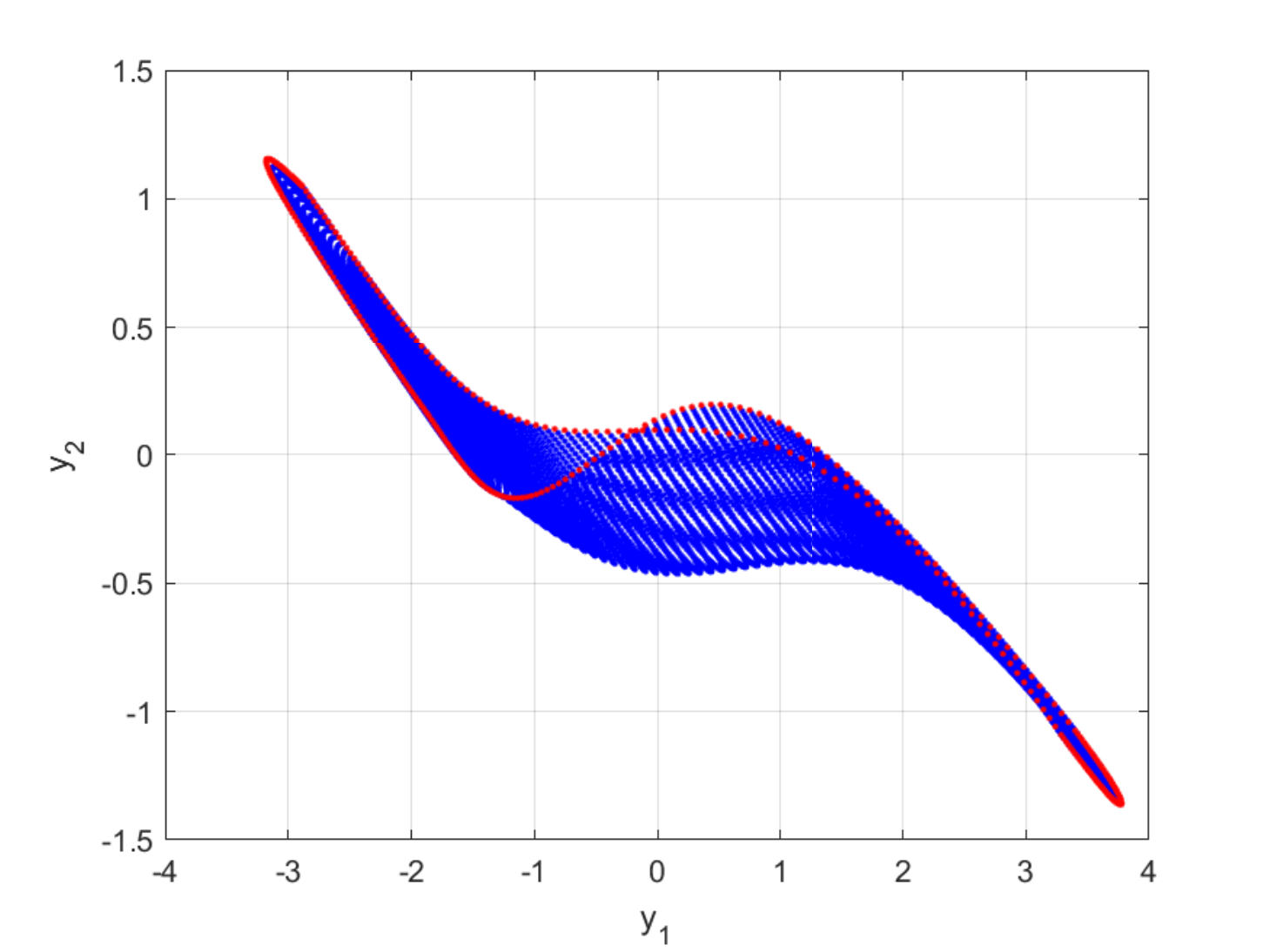}
%\caption{fig1}
\label{illu_eps2}
\end{minipage}%
}%
\subfigure[Set \textcolor{blue}{$\mathcal{A}$}]{
\begin{minipage}[t]{0.32\linewidth}
\centering
\includegraphics[width = 1.7in, height=1.2in]{ 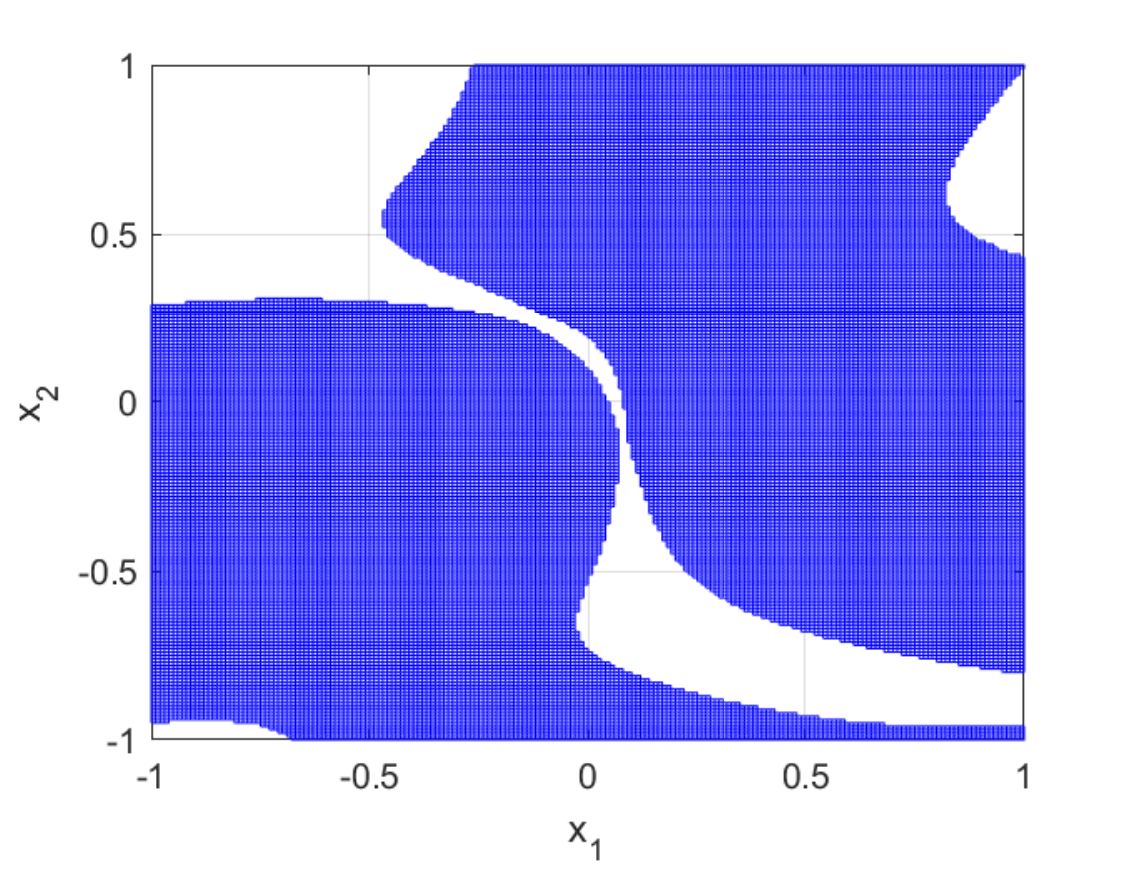}
%\caption{fig1}
\label{illu_eps3}
\end{minipage}%
}%
\subfigure[\textcolor{blue}{$\Omega(\mathcal{X}_{in})$}; \textcolor{red}{$\Omega(\overline{\mathcal{X}_{in}\setminus \mathcal{A}})$}]{
\begin{minipage}[t]{0.32\linewidth}
\centering
\includegraphics[width = 1.7in, height=1.2in]{ 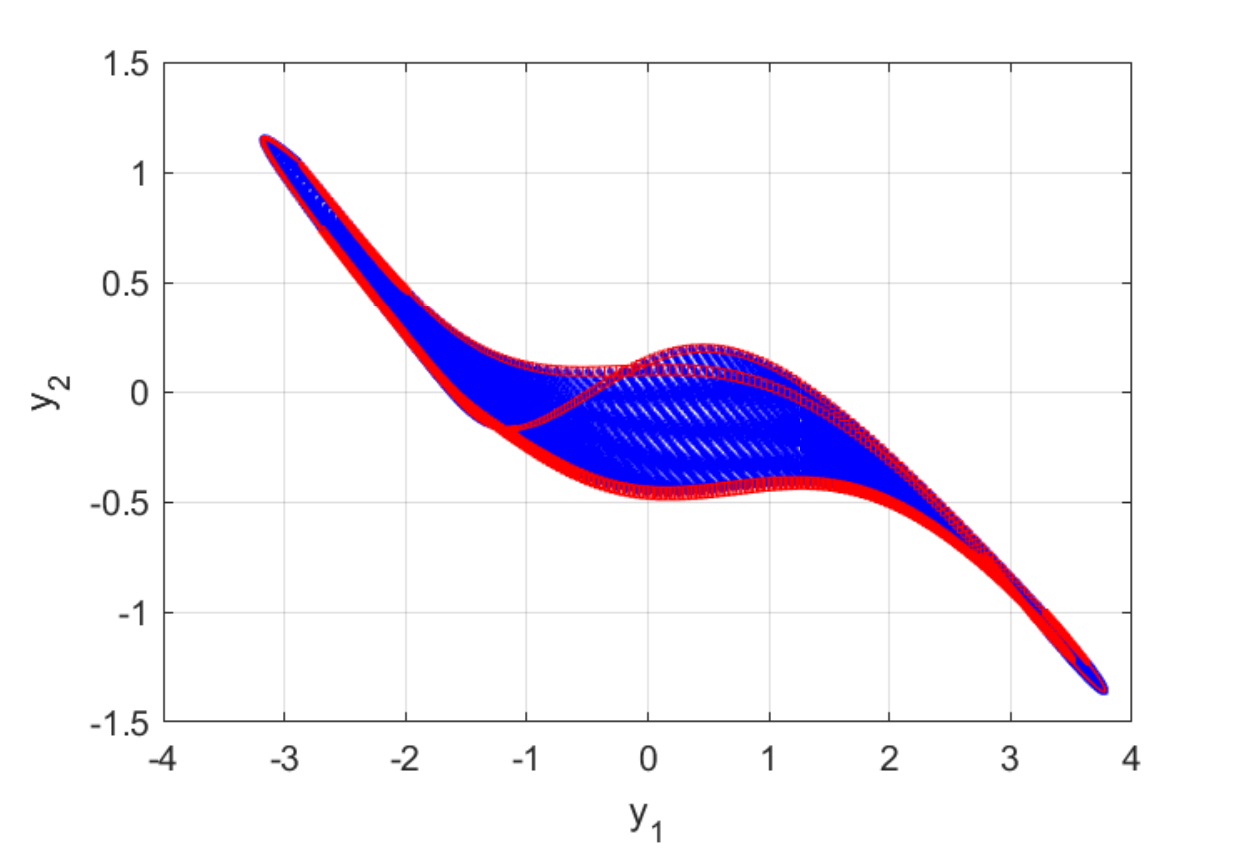}
%\caption{fig2}
\label{illu_eps4}
\end{minipage}%
}%
\centering
\caption{Illustrations on Example \ref{ex2}, \ref{ex3} and \ref{ex4}}
\label{ex234_figure}
\end{figure}

It is interesting to find that the safety verification in the sense of Definition \ref{safety} can be addressed by performing reachability analysis on a subset of the input set $\mathcal{X}_{in}$. This subset is obtained via removing subsets in the input set $\mathcal{X}_{in}$, which features the NN with the homeomorphism property.
\begin{theorem}
\label{wadiao}
Let $\mathcal{A}\subseteq \mathcal{X}_{in}$ and $\mathcal{A}\cap \partial \mathcal{X}_{in}=\emptyset$, and $\bm{N}(\cdot): \mathcal{A}\rightarrow \mathcal{R}(\mathcal{A})$ be a homeomorphism with respect to the input set $\mathcal{A}$. Then, if the output reachable set of the closure of the set $\mathcal{X}_{in}\setminus \mathcal{A}$ is a subset of the safe set $\mathcal{X}_s$, i.e., $\mathcal{R}(\overline{\mathcal{X}_{in}\setminus \mathcal{A}})\subseteq \mathcal{X}_{s}$, the safety property that $\forall \bm{x}_0\in \mathcal{X}_{in}. \ \bm{N}(\bm{x}_0)\in \mathcal{X}_s$
holds.
\end{theorem}
\begin{proof}
Obviously, if $\mathcal{R}(\mathcal{A})\subseteq \mathcal{X}_{s}$ and $\mathcal{R}(\overline{\mathcal{X}_{in}\setminus \mathcal{A}})\subseteq \mathcal{X}_s$, the safety property that $\forall \bm{x}_0\in \mathcal{X}_{in}.\  \bm{N}(\bm{x}_0)\in \mathcal{X}_s$
holds.

According to Theorem \ref{sound}, we have that if $\mathcal{R}(\partial \mathcal{A}) \subseteq \mathcal{X}_{in}$, 
the safety property that $\forall \bm{x}_0\in \mathcal{A}.\  \bm{N}(\bm{x}_0)\in \mathcal{X}_s$ holds. 

According to the condition that $\mathcal{A}\subseteq \mathcal{X}_{in}$ and $\mathcal{A}\cap \partial \mathcal{X}_{in}=\emptyset$, we have that $\mathcal{A} \subseteq \mathcal{X}_{in}^{\circ}$ and thus
$\partial \mathcal{A}\subseteq \overline{\mathcal{X}_{in}\setminus \mathcal{A}}$. Therefore,  $\mathcal{R}(\overline{\mathcal{X}_{in}\setminus \mathcal{A}})\subseteq \mathcal{X}_s$ implies that $\forall \bm{x}_0\in \mathcal{X}_{in}. \ \bm{N}(\bm{x}_0)\in \mathcal{X}_s$.
The proof is completed.
\end{proof}

Theorem \ref{wadiao} tells us that it is still possible to use a subset of the input set for addressing the problem in Definition \ref{safety}, even if the given NN is not a homeomorphism with respect to $\mathcal{X}_{in}$. This is shown in Example \ref{ex4}.
\begin{example}
    \label{ex4}
    Consider the situation in Example \ref{ex3} again. If the entire input set is used for computations, all of $4\times 10^4$ small intervals participate in calculations. However, 
    Theorem \ref{wadiao} tells us that only 9071 intervals (i.e., subset $\overline{\mathcal{X}_{in}\setminus \mathcal{A}}$) are needed, which is much smaller than $4\times 10^4$. The computation results based on interval arithmetic are illustrated in Fig. \ref{illu_eps4}. It is noting that $9071$ intervals rather than $8527(=4\times 10^4 -31473)$ intervals are used since some intervals, which have non-empty intersection with the boundary of the input set $\mathcal{X}_{in}$ (since Theorem \ref{wadiao} requires $\mathcal{A}\cap \partial \mathcal{X}_{in} = \emptyset$), should participate in calculations.

\iffalse
    \begin{figure}[htb!]
\center
\includegraphics[width=1.8in]{ illu_31.pdf} 
%and  each component of $\bm{c}$ is in $[-10^2,10^2]$ 
\caption{blue region-- $\Omega(\mathcal{X}_{in})$; red region--$\Omega(\overline{\mathcal{X}_{in}\setminus \mathcal{A}})$}
\label{illu_eps4}
\end{figure}
\fi
\end{example}

\begin{remark}
According to Theorem \ref{wadiao}, we can also observe that the boundary of the output reachable set $\mathcal{R}(\mathcal{X}_{in})$ is included in the output reachable set of the input set $\overline{\mathcal{X}_{in}\setminus \mathcal{A}}$, i.e., $\partial \mathcal{R}(\mathcal{X}_{in}) \subseteq \mathcal{R}(\overline{\mathcal{X}_{in}\setminus \mathcal{A}})$. This can also be visualized in Fig. \ref{illu_eps4}. Consequently, this observation may open new research directions of addressing various problems of NNs \cite{ghorbani2019interpretation}.  For instance, it may facilitate the generation of adversarial examples, which are inputs causing the NN to falsify the safety property, and the characterization of decision boundaries of NNs, which are a surface that separates data points belonging to different class labels.  
\end{remark}

Therefore, we arrive at an algorithm for safety verification of non-invertible NNs, which is formulated in Algorithm \ref{alg: iNNs1}.
\begin{algorithm}[tb]
\caption{Safety Verification Framework for Non-Invertible NNs}
\label{alg: iNNs1}
\textbf{Input}: a non-invertible NN $\bm{N}(\cdot): \mathbb{R}^n \rightarrow \mathbb{R}^n$, an input set $\mathcal{X}_{in}$ and a safe set $\mathcal{X}_s$.\\
\textbf{Output}: \textbf{Safe} or \textbf{Unknown}.
\begin{algorithmic}[1] %[1] enables line numbers
\STATE determine a subset $\mathcal{A}$ of the set $\mathcal{X}_{in}$ such that $\bm{N}(\cdot): \mathbb{R}^n \rightarrow \mathbb{R}^n$ is a homeomorphism with respect to it;
\STATE apply existing methods to compute an over-approximation $\Omega(\overline{\mathcal{X}_{in}\setminus  \mathcal{A}})$;
\IF {$\Omega(\overline{\mathcal{X}_{in}\setminus  \mathcal{A}})\subseteq \mathcal{X}_s$} 
\STATE return \textbf{Safe}
%\STATE $\mathcal{B}_{out}=T(\mathbb{N}, \mathcal{B}_{in})$.\label{Line 2}
%\STATE $\widehat{\mathcal{X}_{out}}=$CVX$(\mathcal{B}_{out})$.\label{Line 3}
\ELSE
\STATE return \textbf{Unknown}
\ENDIF
%\STATE $\mathcal{X}_{in}=\widehat{\mathcal{X}_{out}}$.
%\STATE \textbf{return} $\widehat{\mathcal{X}_{out}}$
\end{algorithmic}
\end{algorithm}

\begin{theorem}[Soundness]
    \label{sound1}
If Algorithm \ref{alg: iNNs1} returns \textbf{\em Safe},  the safety property that $\forall \bm{x}_0\in \mathcal{X}_{in}. \ \bm{N}(\bm{x}_0) \in \mathcal{X}_s$
holds.    
\end{theorem}
\begin{proof}
This conclusion can be assured by Theorem \ref{wadiao}.
\end{proof}

\begin{remark}
    The set-boundary reachability method can also be applied to intermediate layers in a given NN, rather than just the input and output layers. Suppose that there exists a sub-NN $\bm{N}'(\cdot):\mathbb{R}^{n'}\rightarrow \mathbb{R}^{n'}$, which maps the input of the $l$-th layer to the output of the $k$-th layer, in the given NNs, and its input set is $\mathcal{X}'_{in}$ which is an over-approximation of the output reachable set of the $(l-1)$-th layer. If $\bm{N}'(\cdot): \mathbb{R}^{n'}\rightarrow \mathbb{R}^{n'}$ is a homeomorphism with respect to $\mathcal{X}'_{in}$, we can use $\partial \mathcal{X}'_{in}$ to compute an over-approximation $\Omega'(\partial \mathcal{X}'_{in})$ of the output reachable set $\{\bm{y}\mid \bm{y}=\bm{N}'(\bm{x}_0),\bm{x}_0\in \partial \mathcal{X}'_{in}\}$; otherwise, we can apply Theorem \ref{wadiao} and compute an over-approximation $\Omega'(\overline{\mathcal{X}'_{in}\setminus \mathcal{A}})$ of the output reachable set $\{\bm{y}\mid \bm{y}=\bm{N}'(\bm{x}_0),\bm{x}_0\in \overline{\mathcal{X}'_{in}\setminus \mathcal{A}}\}$. In case that the $k$-th layer is not the output layer of the NN, we need to construct a simply connected set, like convex polytope, zonotope or interval, to cover $\Omega'(\partial \mathcal{X}'_{in})$ or $\Omega'(\overline{\mathcal{X}'_{in}\setminus \mathcal{A}})$ for the subsequent layer-by-layer propagation. This set is an over-approximation of the output reachable set of the $k$-th layer, according to Lemma 1 in \cite{xue2016under}.
\end{remark}

\begin{remark}
    Any existing over-approximate  reachability methods such as interval arithmetic- \cite{wang2018efficient}, zonotopes- \cite{singh2018fast}, star sets \cite{tran2019star} based methods, which are suitable for given NNs, can be used to compute the involved over-approximations, i.e., $\Omega(\partial \mathcal{X}_{in})$ and $\Omega(\overline{\mathcal{X}_{in}\setminus \mathcal{A}})$, in Algorithm \ref{alg: iNNs} and \ref{alg: iNNs1}.
\end{remark}

\section{Experiment}
In this section, several examples of NNs are used to demonstrate the performance of the proposed set-boundary reachability method for safety verification. Experiments are conducted on invertible NNs and non-invertible ones respectively. Recall that the proposed set-boundary method is applicable for any reachability analysis algorithm based on set representation, resulting in tighter and verifiable over-approximations when existing approaches fail.  Thus, we compare  the set-boundary method versus the entire set one on some existing reachability tools in terms of efficiency.

\noindent\textbf{Experiment Setting.} All the experiments herein are run on MATLAB 2021a with Intel (R) Core (TM) i7-10750H CPU@2.60 GHz and RAM 16 GB. The codes and models are available from \url{https://github.com/laode2022/BoundaryNN}.

\subsection{Experiments on Invertible NNs}
In this subsection, we carry out some examples involving neural ODEs and invertible feedforward neural networks.
%The separate comparisons also stem from that the invertibility of neural ODEs is ensured by ODEs and that of other invertible NNs is justified by aforementioned non-zero Jacobian determinants in Subsection \ref{Sec-inn}.}

\begin{figure}[htbp]
\centering
\oomit{\subfigure[Reachable Sets]{
\begin{minipage}[t]{0.5\linewidth}
\centering
\includegraphics[width=1in]{ 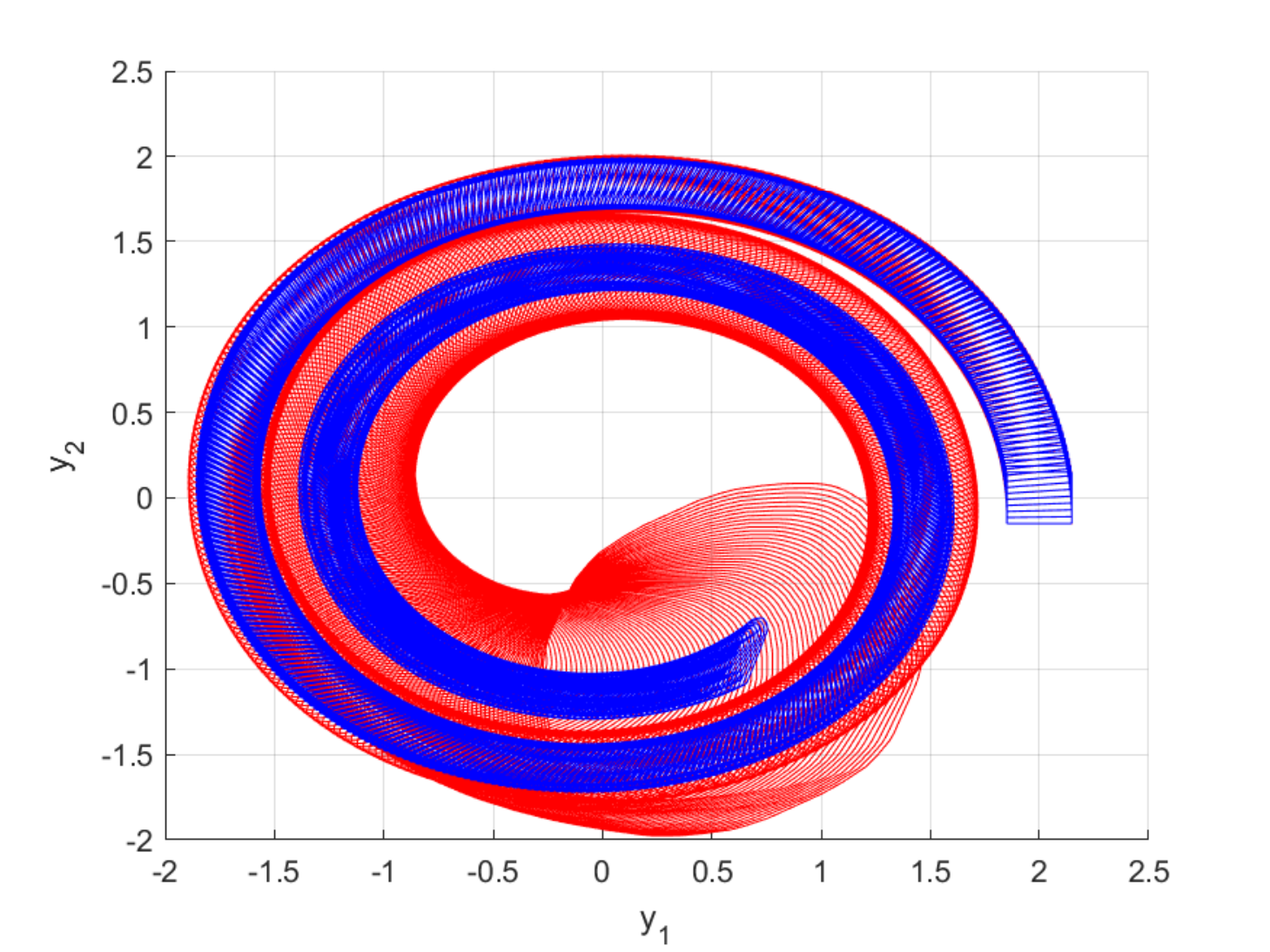}
%\caption{fig1}
\label{spiral_non}
\end{minipage}%
}%
\subfigure[Verification:  \textcolor{blue}{Safe}, \textcolor{red}{Unknown}]{
\begin{minipage}[t]{0.5\linewidth}
\centering
\includegraphics[width=1in]{ 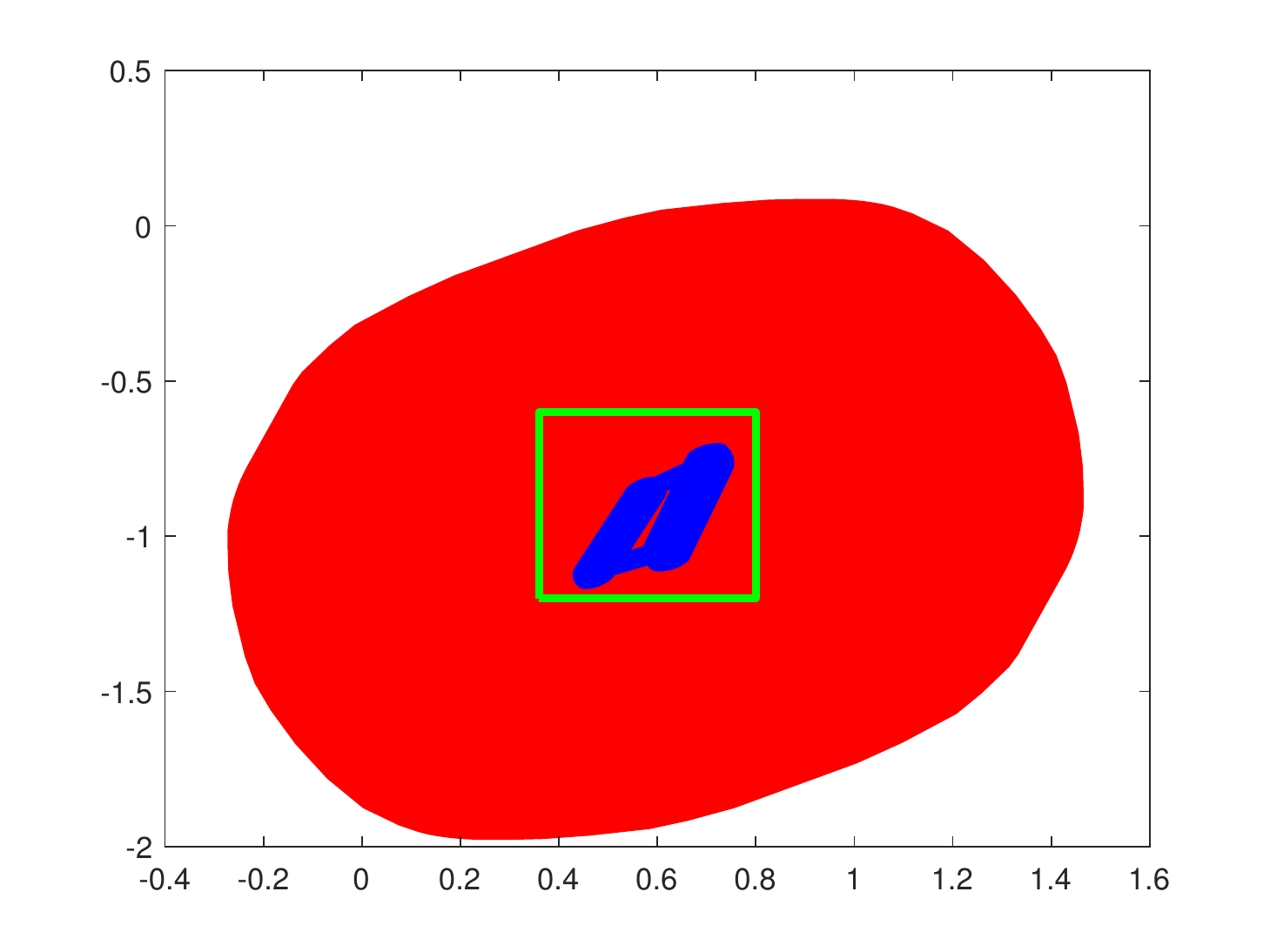}
%\caption{fig2}
\label{spiral_non_conc}
\end{minipage}%
}%
\\}
\subfigure[\textcolor{blue}{$4\times4$} Vs. \textcolor{red}{$4^2$}; \textcolor{blue}{Safe}, \textcolor{red}{Unknown}]{
\begin{minipage}[t]{0.5\linewidth}
\centering
\includegraphics[width=1.6in]{ 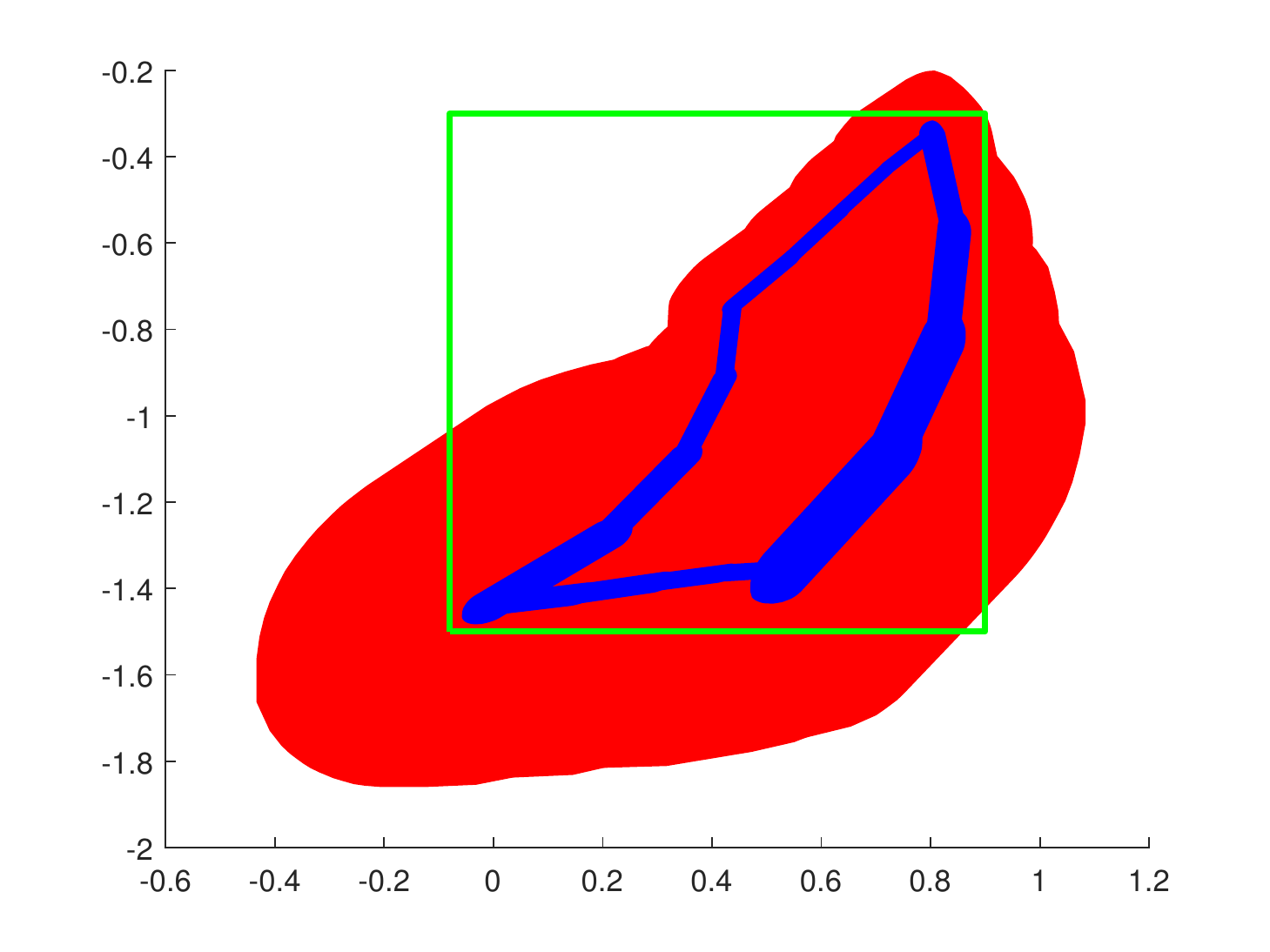}
%\caption{fig2}
\label{spiral_2}
\end{minipage}%
}%
\subfigure[\textcolor{blue}{$4\times4$} Vs. \textcolor{red}{$7^2$}; \textcolor{blue}{Safe}, \textcolor{red}{Safe}]{
\begin{minipage}[t]{0.5\linewidth}
\centering
\includegraphics[width=1.6in]{ 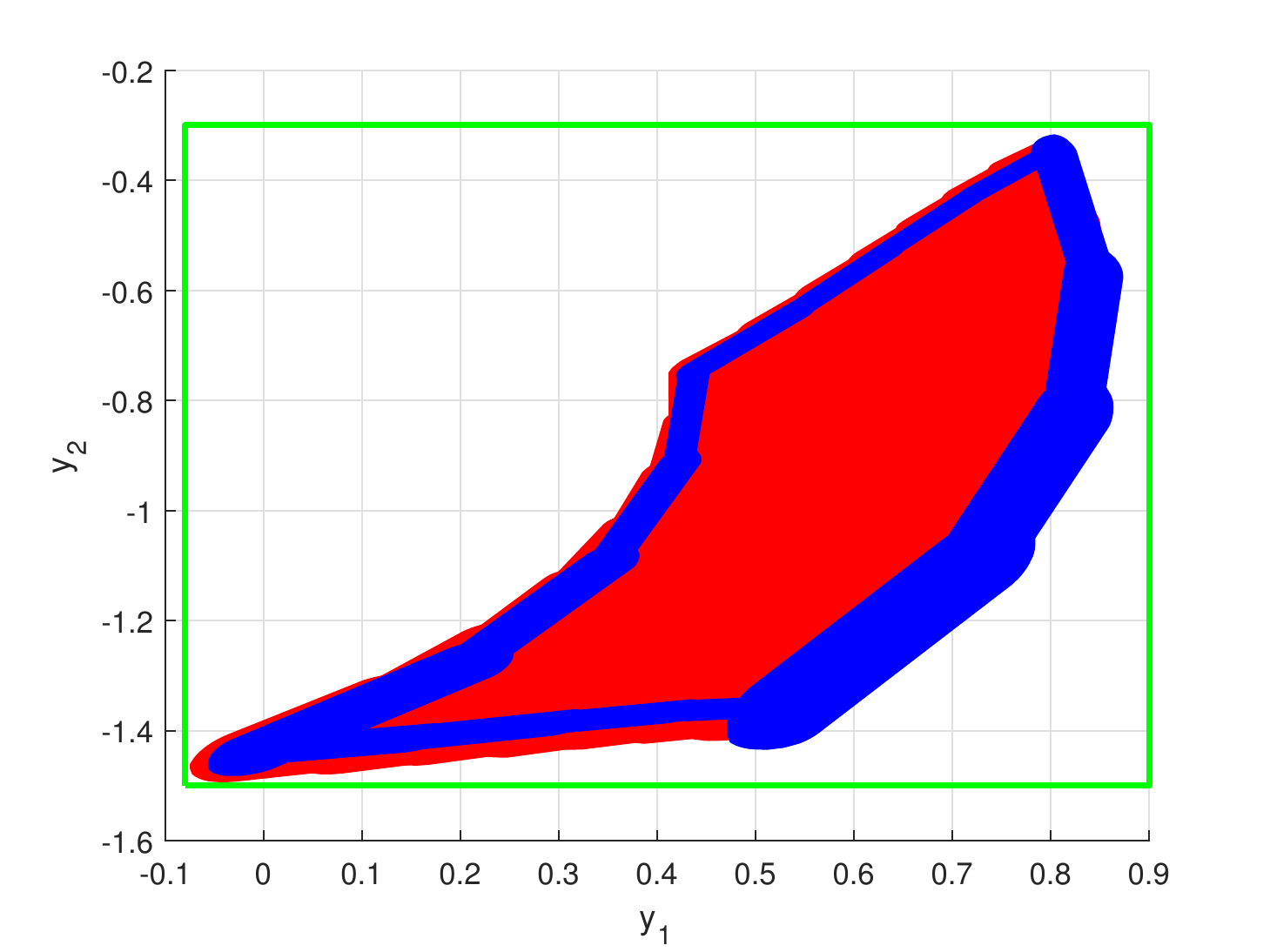}
%\caption{fig2}
\label{spiral_3}
\end{minipage}
}%%
\\
\centering
\subfigure[Reachable Sets]{
\begin{minipage}[t]{0.5\linewidth}
\centering
\includegraphics[width=1.6in]{ 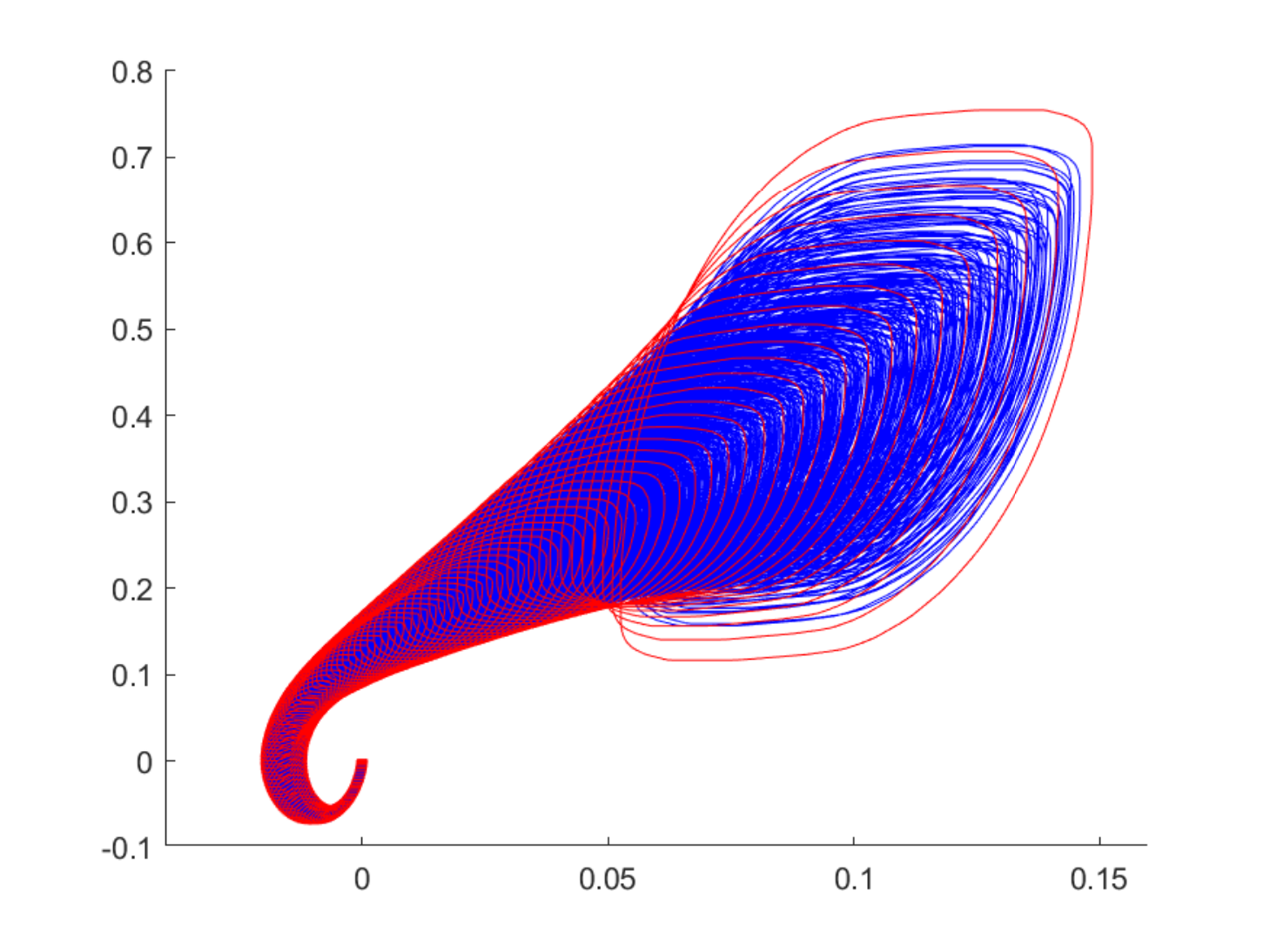}
%\caption{fig1}
\label{FPA}
\end{minipage}%
}%
\subfigure[Verification:  \textcolor{blue}{Safe}, \textcolor{red}{Unknown}]{
\begin{minipage}[t]{0.5\linewidth}
\centering
\includegraphics[width=1.6in]{ 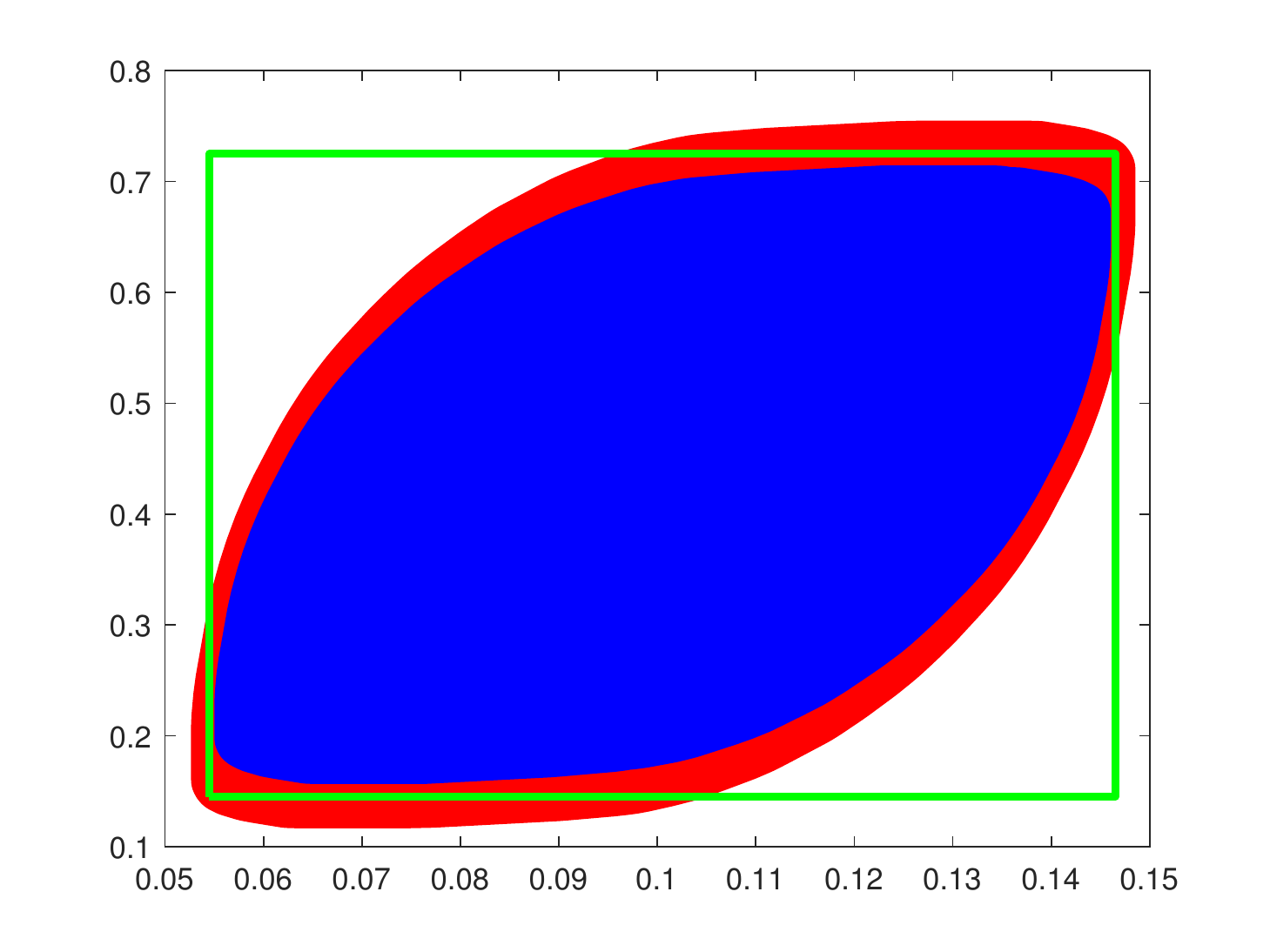}
%\caption{fig2}
\label{FPA_conc}
\end{minipage}%
}%
\caption{Verification on $\bm{N}_1$, $\bm{N}_2$.
\textcolor{blue}{$\Omega(\partial \mathcal{X}_{in})$}; \textcolor{red}{$\Omega(\mathcal{X}_{in})$}; \textcolor{green}{$\partial \mathcal{X}_{s}$}}
\label{node}
\end{figure}

\noindent{\textbf{Neural ODEs.}} We experiment on two widely-used neural ODEs in  \cite{manzanas2022reachability}, which are respectively a nonlinear 2-dimensional spiral \cite{chen2018neural} with the input set $\mathcal{X}_{in}=[1.5,2.5]\times [-0.5,0.5]$ and the safe set $\mathcal{X}_s=[-0.08, 0.9]\times [-1.5,-0.3]$  and a 12-dimensional controlled cartpole \cite{gruenbacher2020lagrangian} with the input set $\mathcal{X}_{in}=[-0.001,0.001]^{12}$ and the safe set $\mathcal{X}_s=[0.0545,0.1465]\times[0.145,0.725]$. For simplicity, we respectively denote them $\bm{N}_1$ and $\bm{N}_2$. Here, we take zonotopes as abstract domains and compare the output reachable sets computed by our set-boundary reachability method and the entire input set based method. The over-approximate reachability analysis is performed on the continuous reachability analyzer CORA toolbox \cite{althoff2015introduction}. %Experiment is set with initial value: [2.0, 0.0], perturbation radius: 0.15, 
When the time horizon is $[0, 6]$ and the time step is 0.01, our set-boundary reachability method for $\bm{N}_1$ returns `\textcolor{blue}{Safe}' when the boundary of the input set is partitioned into $16$ equal subsets, with the computation time being about 220.83 seconds. However, the entire input set based method returns `\textcolor{red}{Unknown}' when the input set is  partitioned into $16$ equal subsets. These color marks for safety verification results also apply to the experiments below.  The safety property is verified until the entire input set is partitioned into $49$ equal subsets. The corresponding computation time is 671.32 seconds. Consequently, the computation time from our set-boundary reachability method is reduced by 67.1\%, compared to the entire input set based method. The computed output reachable sets for $\bm{N}_1$ are displayed in Fig. \ref{spiral_2} and \ref{spiral_3}. When the time horizon is $[0.0, 1.1]$ and the time step is 0.01, the computed output reachable sets for $\bm{N}_2$ are displayed in Fig. \ref{FPA}. Fig. \ref{FPA_conc} shows the reachable sets at the time instant $t=1.1$.

\noindent\textbf{Feedforward Neural Networks.} Rather than considering neural ODEs, we instead take more general invertible NNs into account.  The invertibility of  NNs used here, i.e., $\bm{N}_3$ and $\bm{N}_4$,  are assured by their Jacobian determinant not being zero. The NN $\bm{N}_3$ is fully connected with Sigmoid activation functions, having an input/output layer with dimension 2 and  10 hidden layers with size 100. The NN $\bm{N}_4$ is similar to $\bm{N}_3$, except that its input/output dimensions are 3.

The results of safety verification of $\bm{N}_3$ and $\bm{N}_4$ are demonstrated in Fig. \ref{2-dim inn} and \ref{3-dim inn1}. Their safe sets $\mathcal{X}_{s}$ are respectively $ [0.914304,0.9143525]\times [0.9508425,0.950896]$ and $[0.2884, 0.289] \times [0.465,0.466]\times[0.5752,0.5762]$, whose boundaries are shown in green color in Fig. \ref{2-dim inn} and \ref{3-dim inn1}.  The input sets in Fig. \ref{2-dim inn} are  $[-0.125,0.125]^2$, $[-0.25,0.25]^2$, $[-0.375,0.375]^2$, $[-0.425,0.425]^2$ and $[-1.0,1.0]^2$ (Fig. \ref{-11_2}-\ref{-11_3}.) respectively  and those of Fig. \ref{3-dim inn1} are $[-0.25,0.25]^3$, $[-0.30,0.30]^3$,  $[-0.375,0.375]^3$. The over-approximate reachability analysis is implemented using DeepZ \cite{singh2018fast}, which is a tool for safety verification of  large feed-forward, convolutional, and residual networks via propagating zonotopes through networks.

\begin{figure}[t]
\centering
\subfigure[$\epsilon=0.125$, \textcolor{blue}{Safe}, \textcolor{red}{Safe}]{
\begin{minipage}[t]{0.46\linewidth}
\centering
\includegraphics[width=1.6in]{ 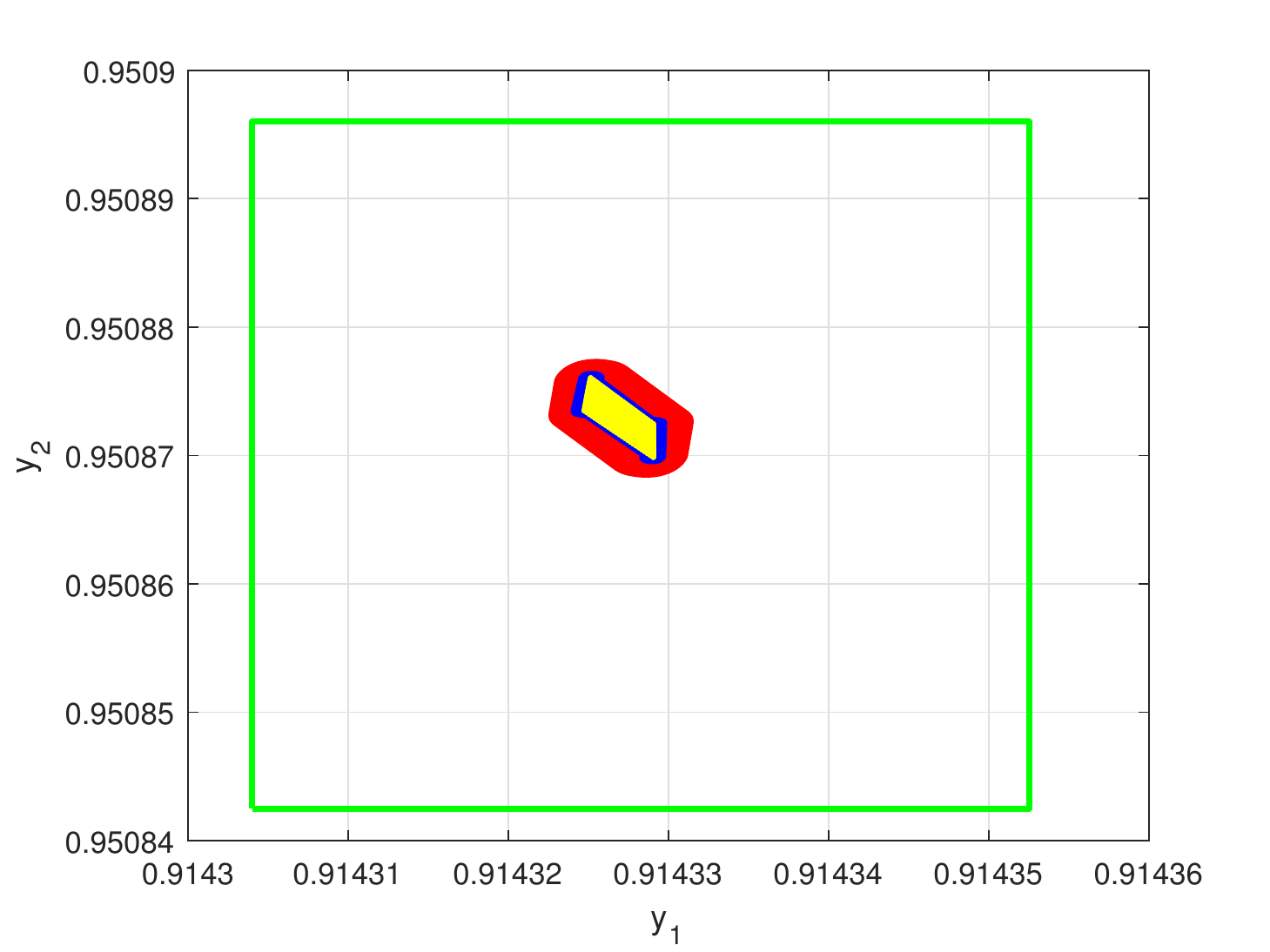}
%\caption{fig1}
\label{0.125}
\end{minipage}%
}%
\subfigure[$\epsilon=0.250$,  \textcolor{blue}{Safe}, \textcolor{red}{Safe}]{
\begin{minipage}[t]{0.46\linewidth}
\centering
\includegraphics[width=1.6in]{ 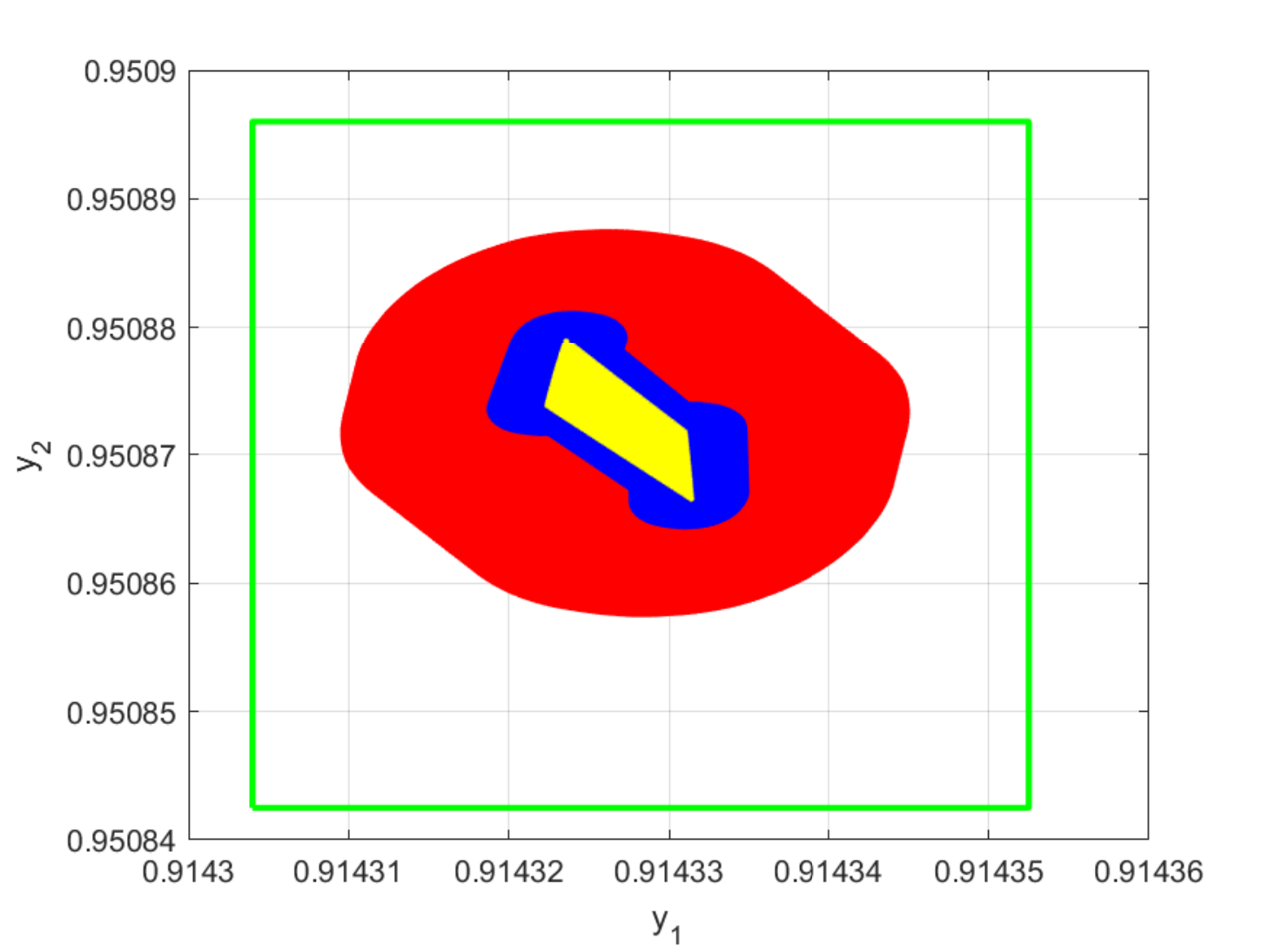}
%\caption{fig2}
\label{0.25}
\end{minipage}%
}%
\\
\subfigure[$\epsilon=0.375$,  \textcolor{blue}{Safe}, \textcolor{red}{Unknown}]{
\begin{minipage}[t]{0.46\linewidth}
\centering
\includegraphics[width=1.6in]{ 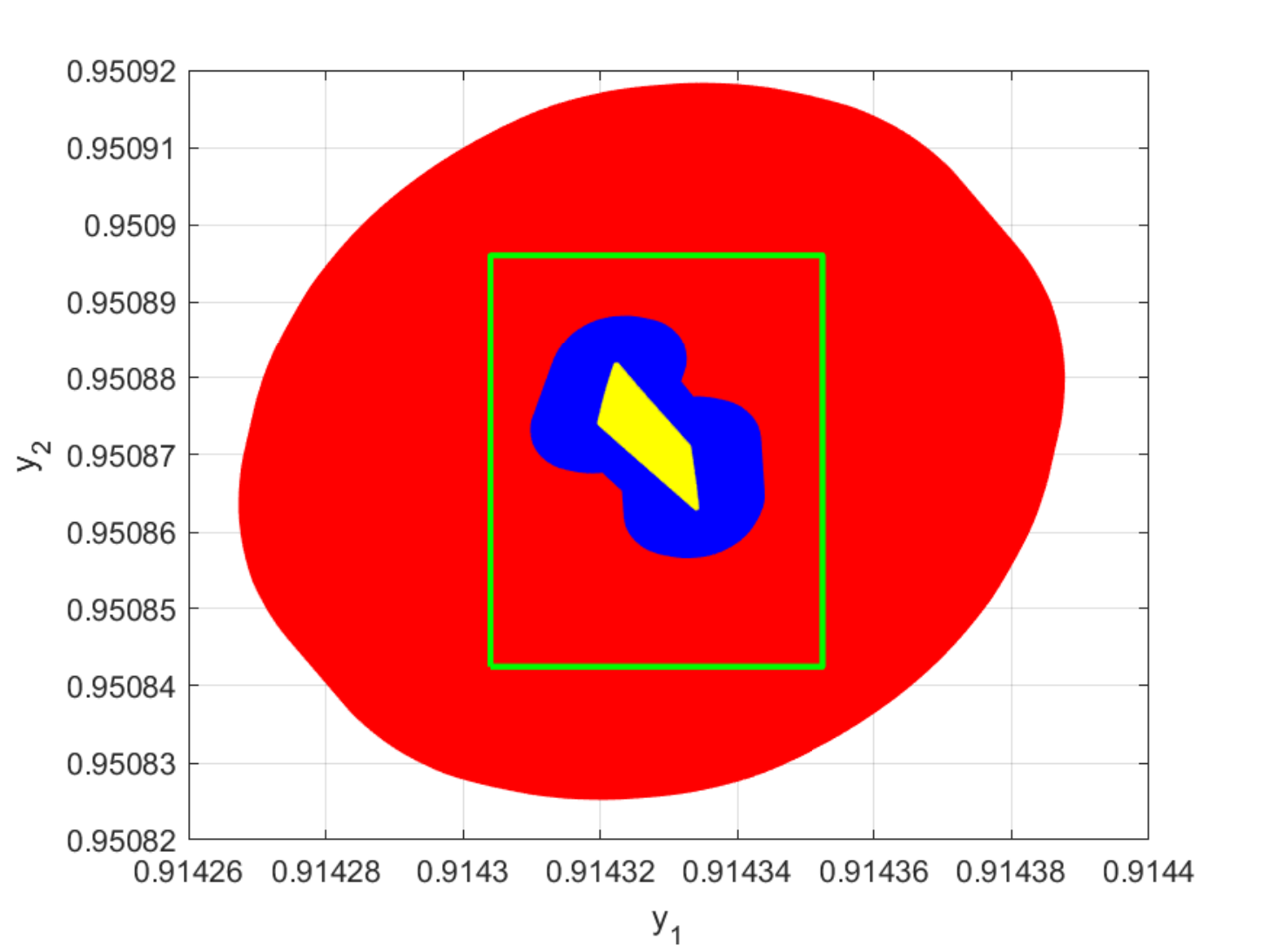}
%\caption{fig1}
\label{0.375}
\end{minipage}%
}%
\subfigure[$\epsilon=0.425$,  \textcolor{blue}{Safe}, \textcolor{red}{Unknown}]{
\begin{minipage}[t]{0.46\linewidth}
\centering
\includegraphics[width=1.6in]{ 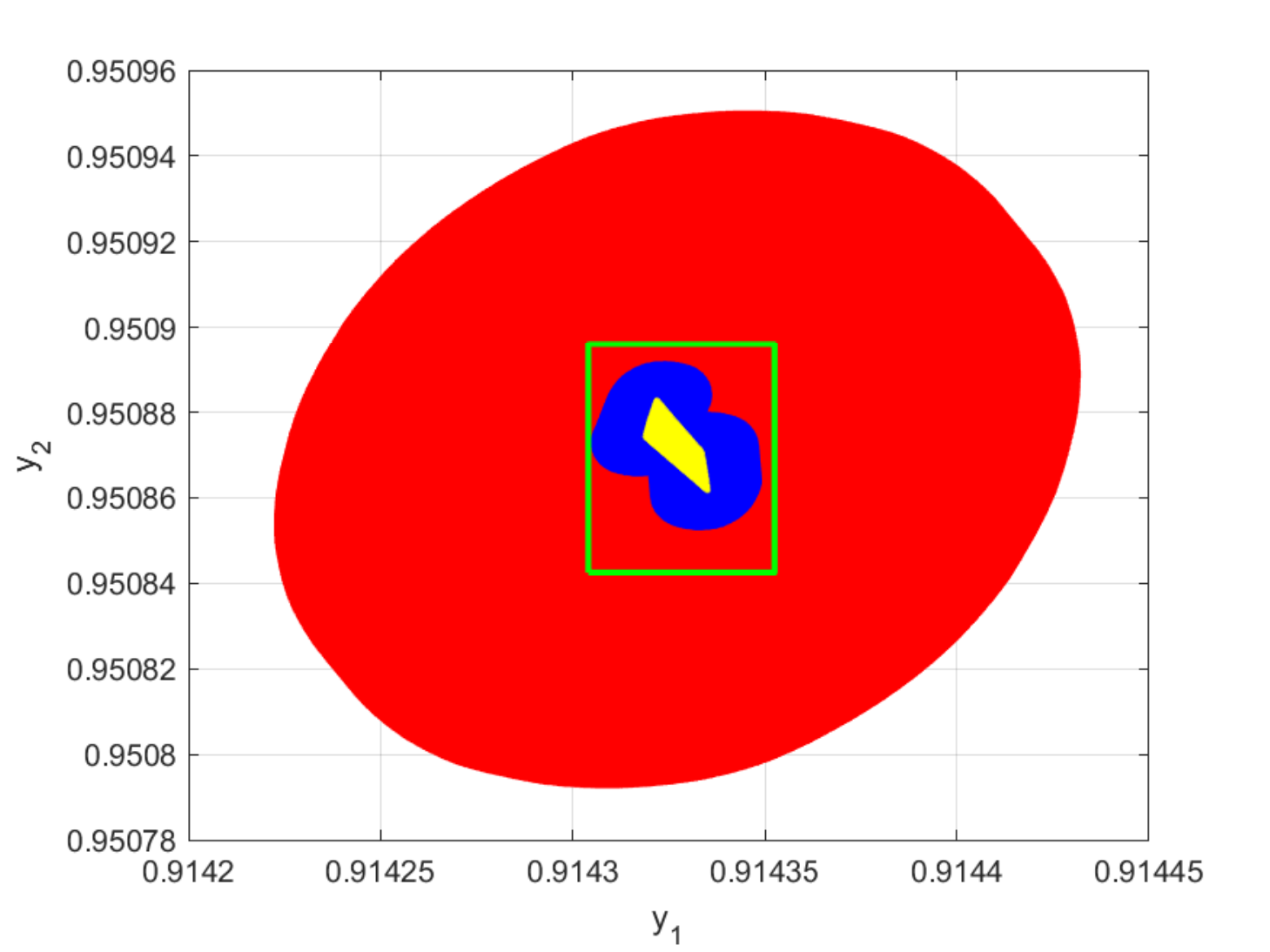}
%\caption{fig2}
\label{0.425}
\end{minipage}%
}%

\oomit{\subfigure[\textcolor{blue}{$4\times4$} Vs. \textcolor{red}{$4^2$}.]{
\begin{minipage}[t]{0.46\linewidth}
\centering
\includegraphics[width=1.6in]{ 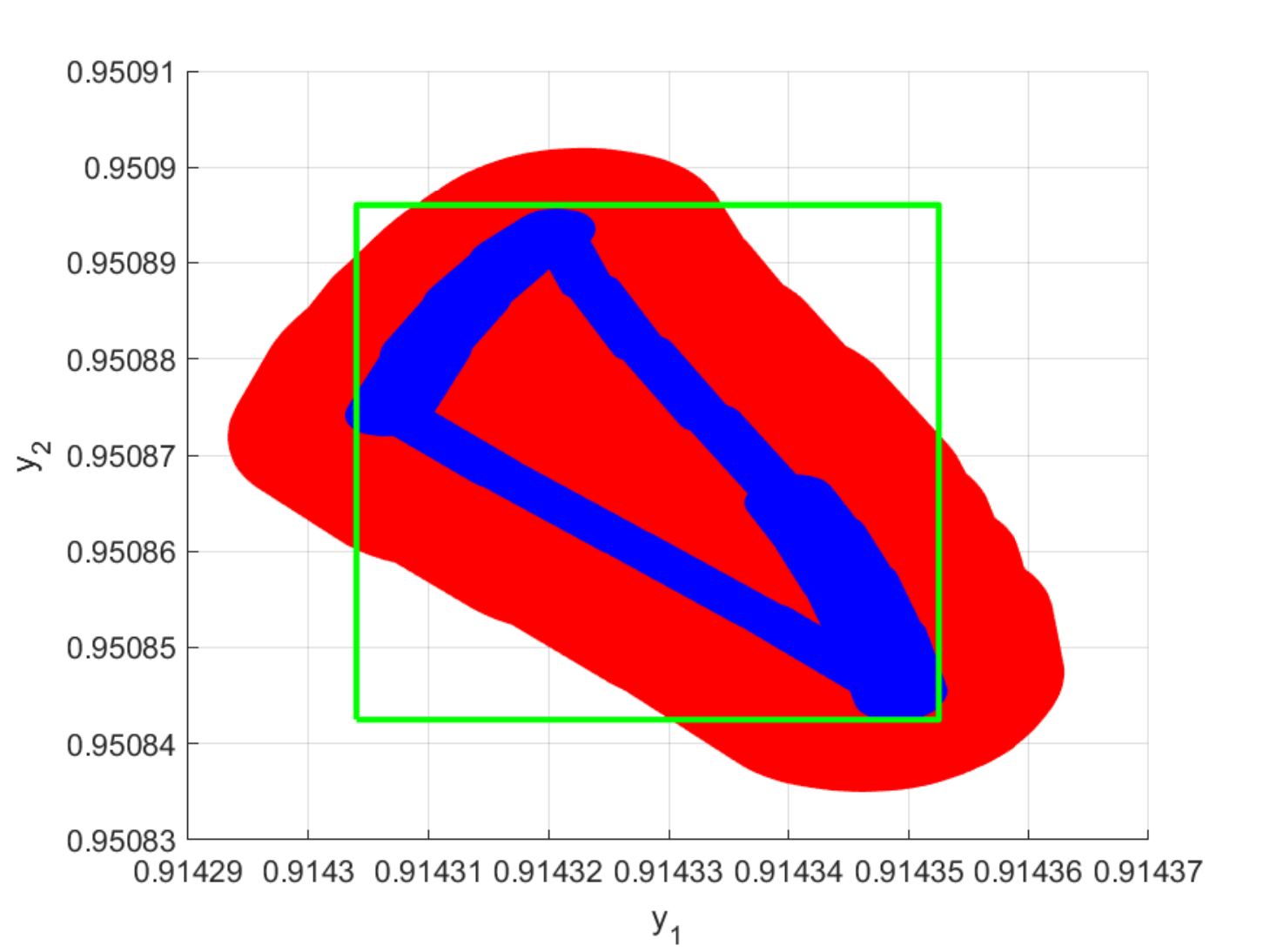}
%\caption{fig1}
\label{-11_1}
\end{minipage}%
}%
}
\subfigure[\textcolor{blue}{$5\times4$} Vs. \textcolor{red}{$5^2$}. \textcolor{blue}{Safe}, \textcolor{red}{Unknown}]{
\begin{minipage}[t]{0.46\linewidth}
\centering
\includegraphics[width=1.6in]{ 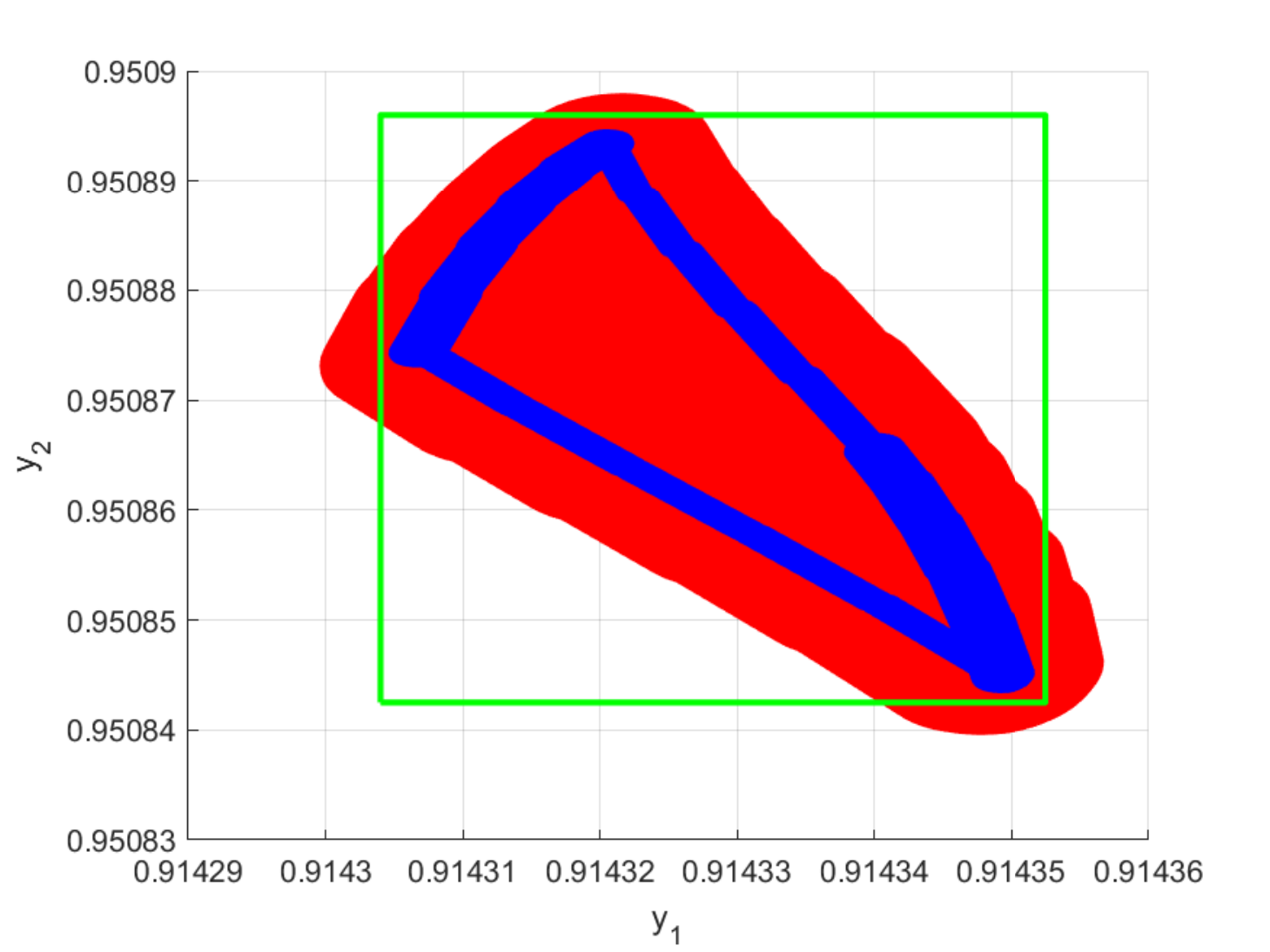}
%\caption{fig2}
\label{-11_2}
\end{minipage}%
}%
\subfigure[\textcolor{blue}{$5\times4$} Vs. \textcolor{red}{$8^2$}. \textcolor{blue}{Safe}, \textcolor{red}{Safe}]{
\begin{minipage}[t]{0.46\linewidth}
\centering
\includegraphics[width=1.6in]{ 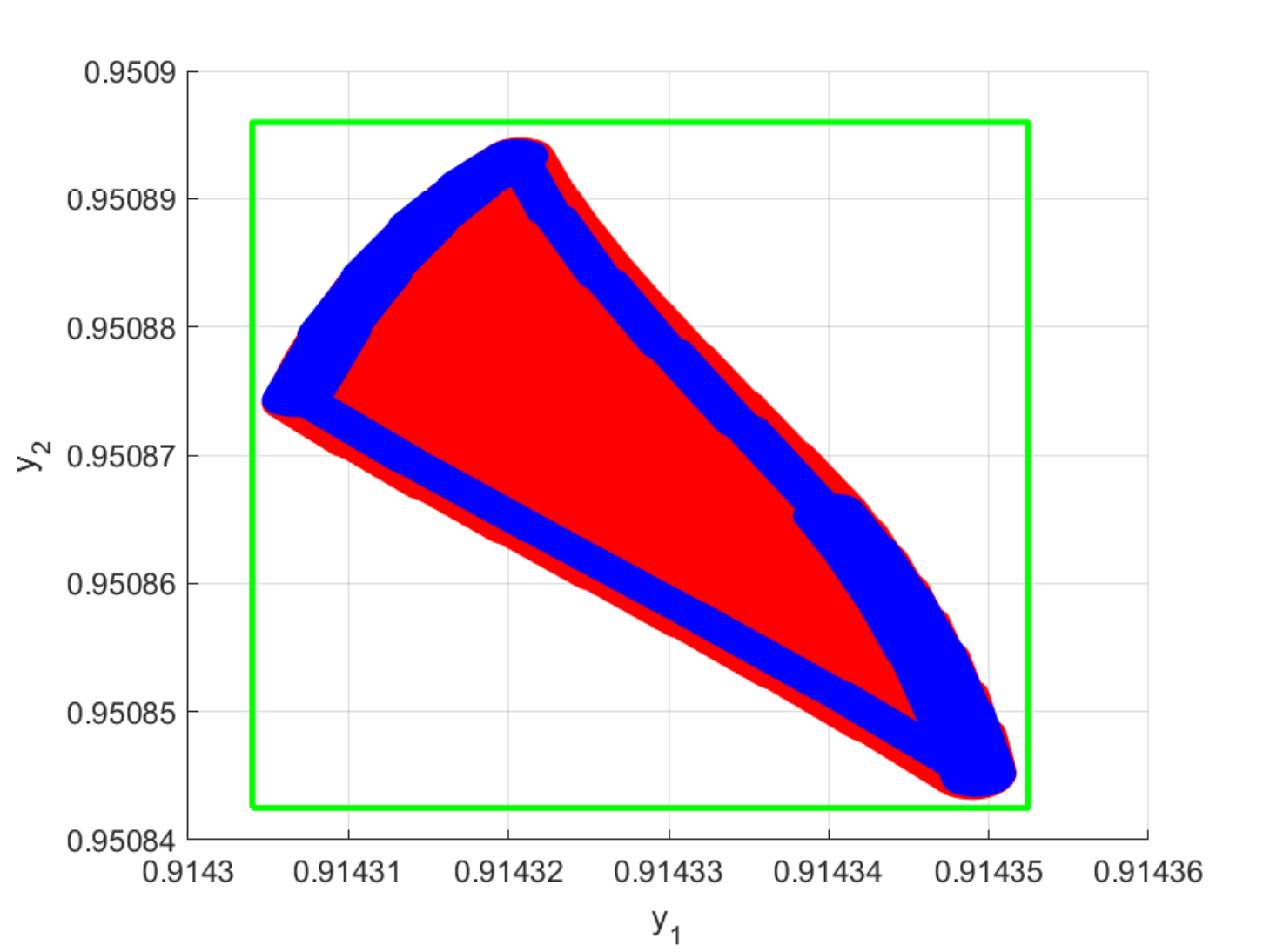}
%\caption{fig2}
\label{-11_3}
\end{minipage}
}%%
\centering
\caption{Safety verification on $\bm{N}_3$.
\textcolor{blue}{$\Omega(\partial \mathcal{X}_{in})$}; \textcolor{red}{$\Omega(\mathcal{X}_{in})$}; \textcolor{green}{$\partial \mathcal{X}_{s}$}}
\label{2-dim inn}
\end{figure}

The output reachable sets from our set-boundary reachability method and the entire set based method are displayed in blue and red in Fig. \ref{2-dim inn} and \ref{3-dim inn1}, respectively. Further, we also show the exact output reachable sets estimated via the Monte-Carlo simulation method in Fig. \ref{2-dim inn} and \ref{3-dim inn1}, which correspond to the yellow regions. The visualized results show that the set-boundary reachability method can generate tighter output reachable sets than the entire set based method. As a result, our set-boundary reachability method can verify the safety properties successfully for all cases. In contrast, the entire set based method fails for large input sets, as shown in Fig. \ref{0.375}, \ref{0.425}, \ref{0.300-1}-\ref{0.300-3} and \ref{0.375-1}-\ref{0.375-3}, since the computed output reachable sets are not included in safe sets. Furthermore, when the safety property cannot be verified with the input set $[-1.0,1.0]^2$, we impose uniform partition operator on both the entire input set and its boundary for verifying the safety property. When the boundary is divided into 20 equal subsets, the safety verification can be verified using our set-boundary reachability method (Fig. \ref{-11_2}) with the computation time of 0.0624 seconds. However, when the entire input set is used, it should be partitioned into 64 equal subsets (Fig. \ref{-11_3}) and the computation time for verification is 0.7405 seconds. Consequently, the computation time from our set-boundary reachability method is reduced by 91.6\%, as opposed to the entire input set based method.

\begin{figure}[t]
\centering
\iffalse
\subfigure[$y_1-y_2$.]{
\begin{minipage}[t]{0.3\linewidth}
\centering
\includegraphics[width=1in]{ homeomorphism_nn2_dim3_0.125_1,2.pdf}
%\caption{fig1}
\end{minipage}%
}%
\subfigure[$y_1-y_3$.]{
\begin{minipage}[t]{0.3\linewidth}
\centering
\includegraphics[width=1in]{ homeomorphism_nn2_dim3_0.125_1,3.pdf}
%\caption{fig2}
\end{minipage}%
}%
\subfigure[$y_2-y_3$]{
\begin{minipage}[t]{0.3\linewidth}
\centering
\includegraphics[width=1in]{ homeomorphism_nn2_dim3_0.125_2,3.pdf}
%\caption{fig2}
\end{minipage}
}%%
\\
$\epsilon=0.125$, \textcolor{blue}{Safe}, \textcolor{red}{Safe}
\\
\fi
\subfigure[$y_1-y_2$]{
\begin{minipage}[t]{0.3\linewidth}
\centering
\includegraphics[width=1.5in]{ 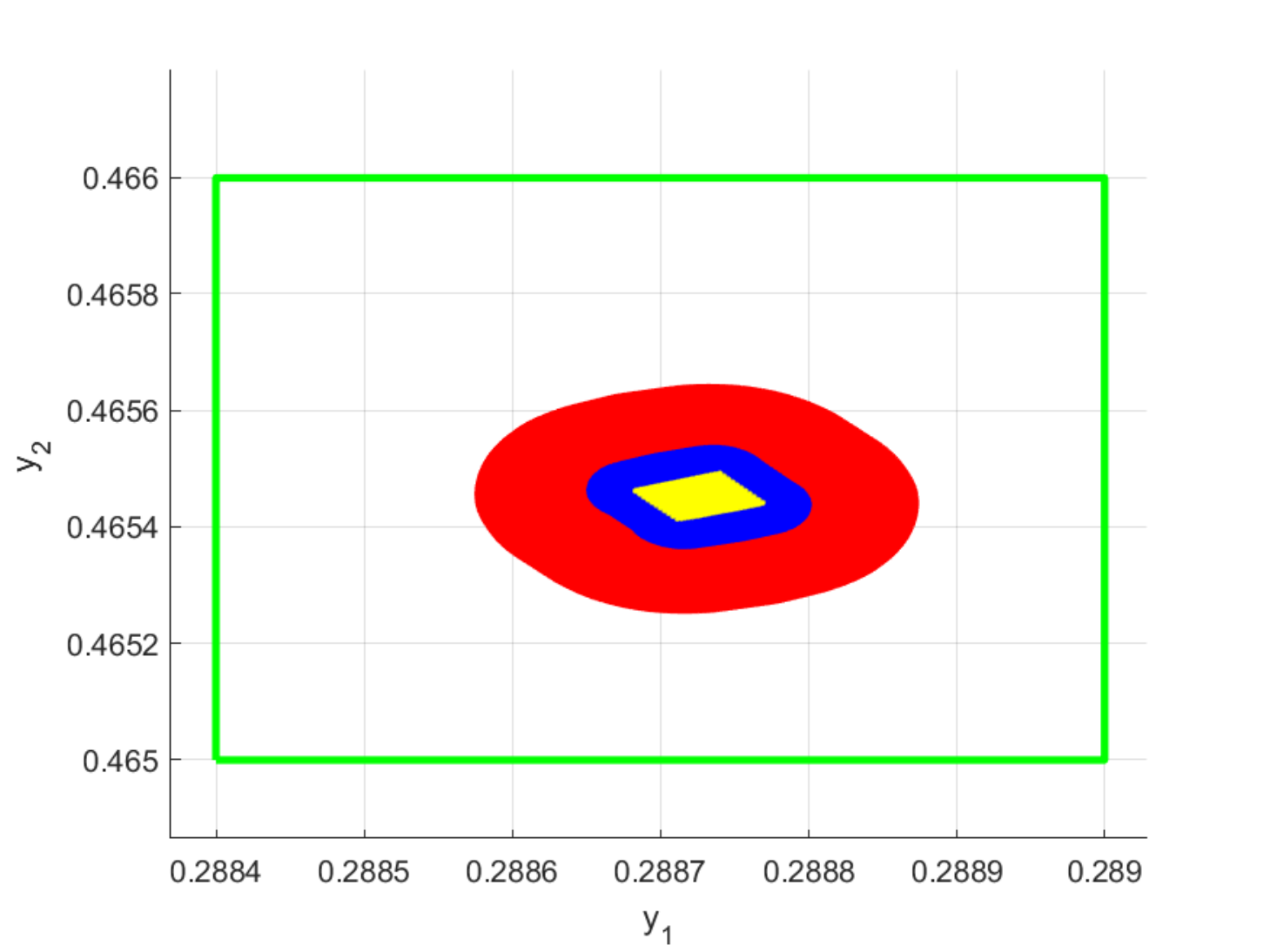}
%\caption{fig1}
\end{minipage}%
}%
\subfigure[$y_1-y_3$]{
\begin{minipage}[t]{0.3\linewidth}
\centering
\includegraphics[width=1.5in]{ 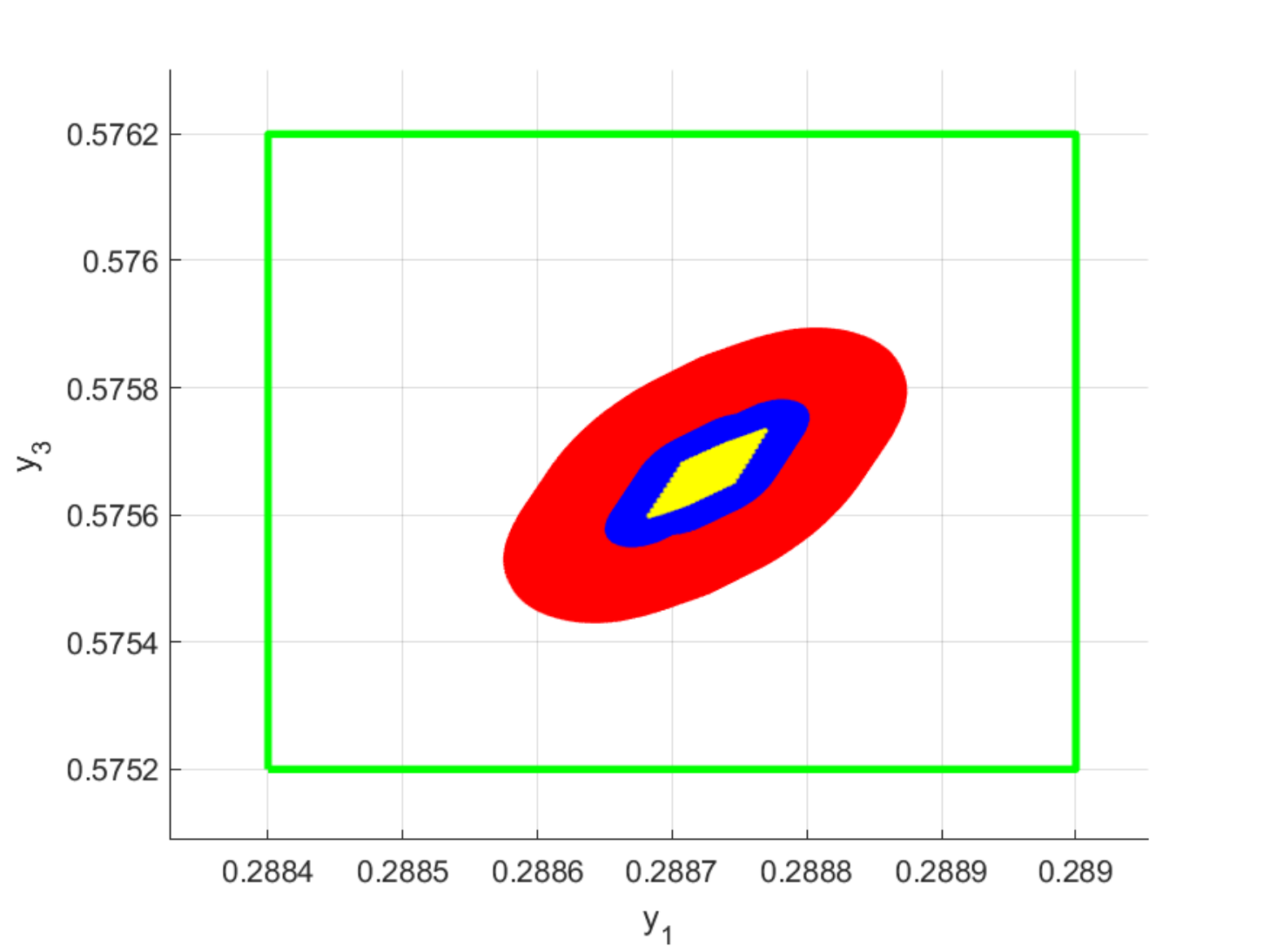}
%\caption{fig2}
\end{minipage}%
}%
\subfigure[$y_2-y_3$]{
\begin{minipage}[t]{0.3\linewidth}
\centering
\includegraphics[width=1.5in]{ 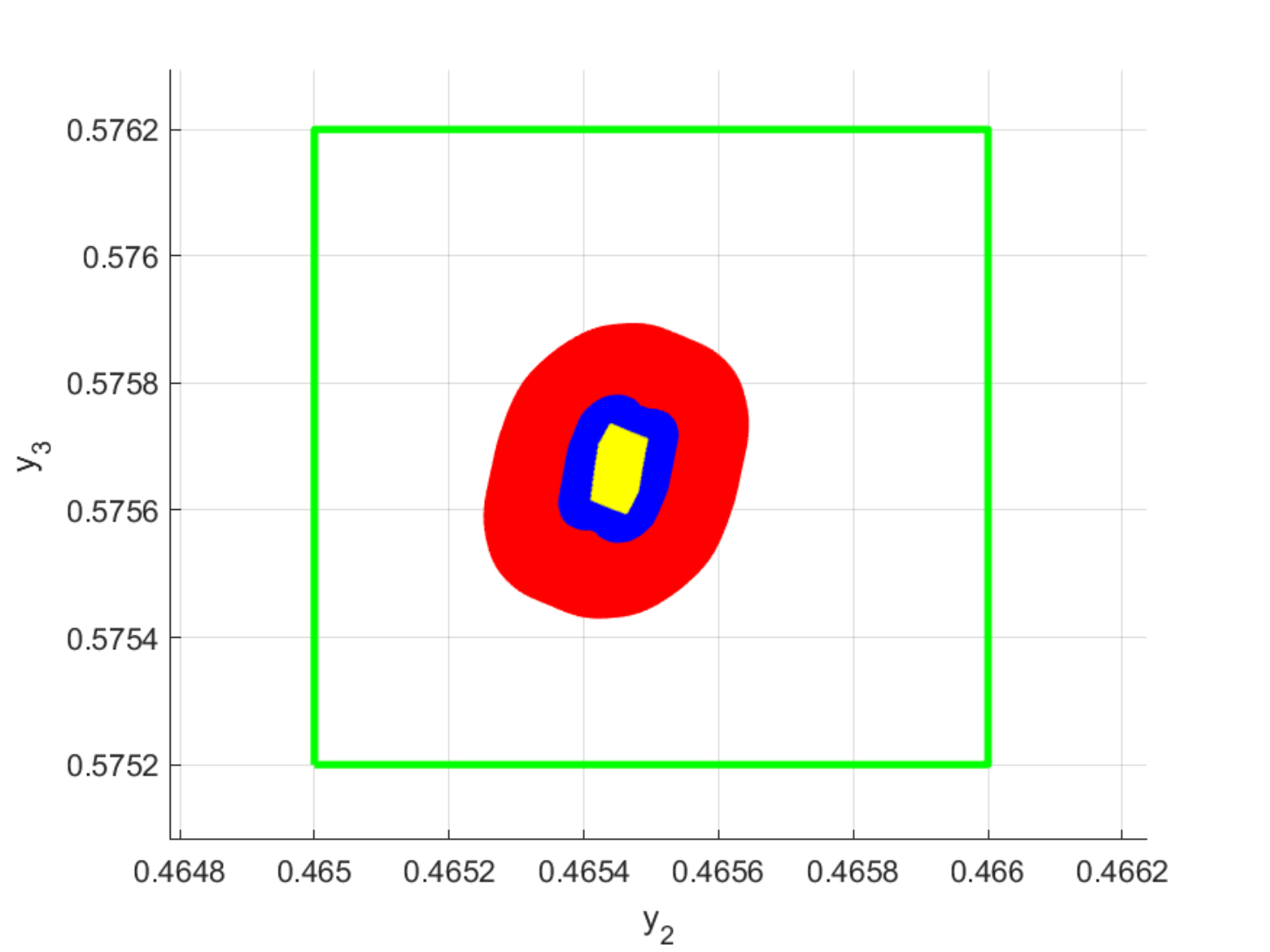}
%\caption{fig2}
\end{minipage}
}%%
\\
$\epsilon=0.250$, \textcolor{blue}{Safe}, \textcolor{red}{Safe}
\\
\subfigure[$y_1-y_2$]{
\begin{minipage}[t]{0.3\linewidth}
\centering
\includegraphics[width=1.5in]{ 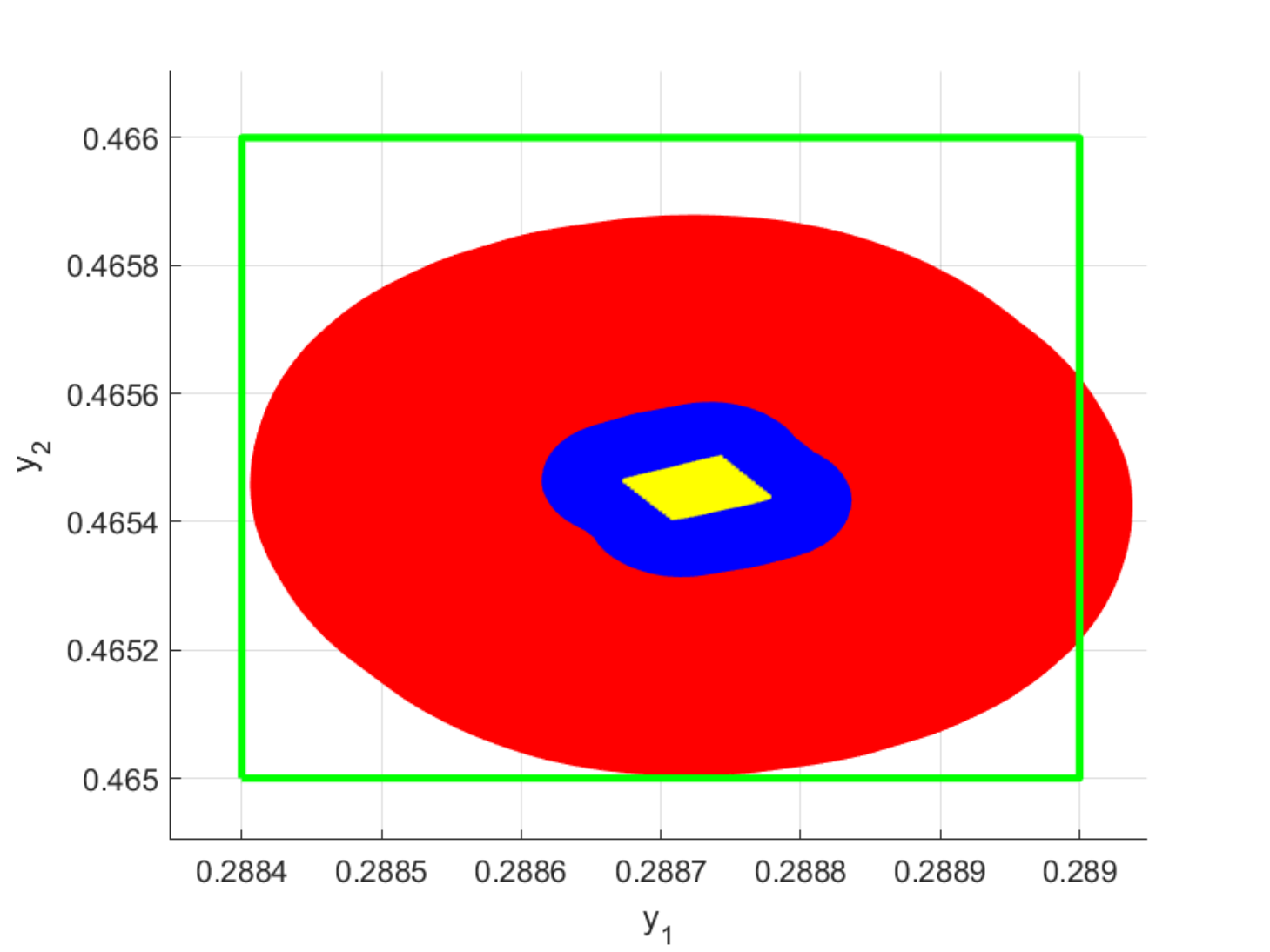}
\label{0.300-1}
\end{minipage}%
}%
\subfigure[$y_1-y_3$]{
\begin{minipage}[t]{0.3\linewidth}
\centering
\includegraphics[width=1.5in]{ 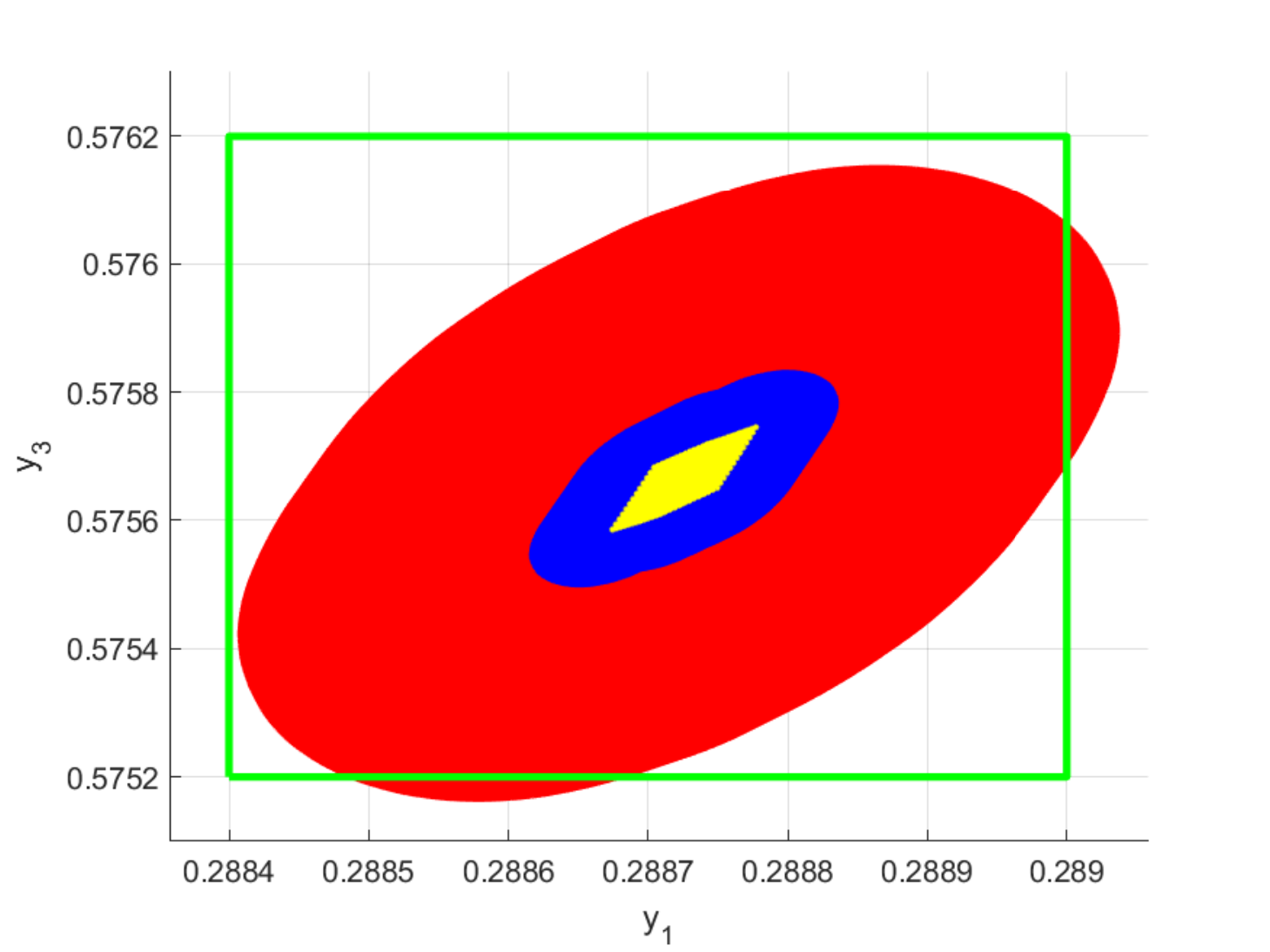}
\label{0.300-2}
\end{minipage}%
}%
\subfigure[$y_2-y_3$.]{
\begin{minipage}[t]{0.3\linewidth}
\centering
\includegraphics[width=1.5in]{ 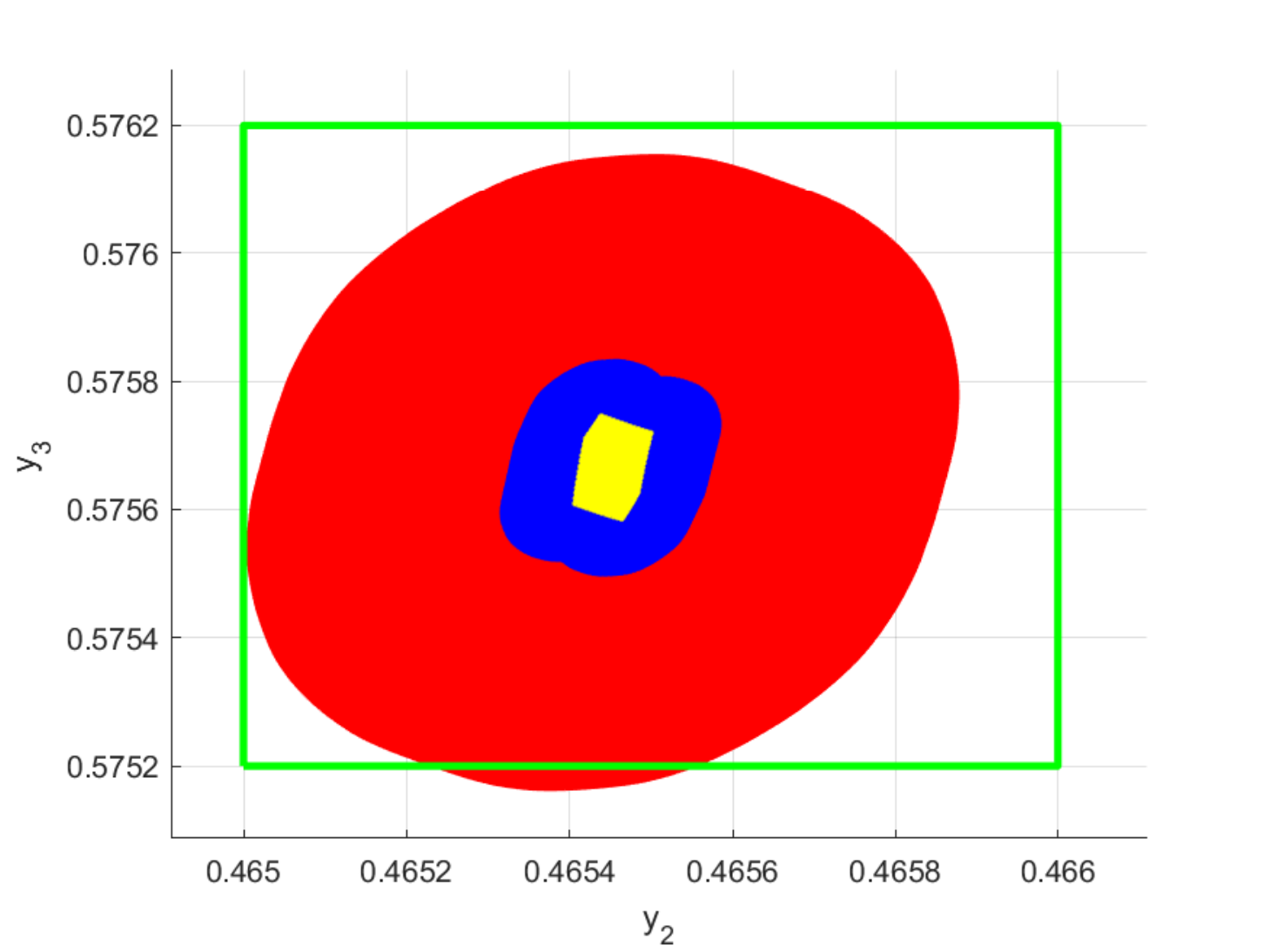}
\label{0.300-3}
\end{minipage}
}%%
\\
$\epsilon=0.300$, \textcolor{blue}{Safe}, \textcolor{red}{Unknown}
\\
\subfigure[$y_1-y_2$]{
\begin{minipage}[t]{0.3\linewidth}
\centering
\includegraphics[width=1.5in]{ 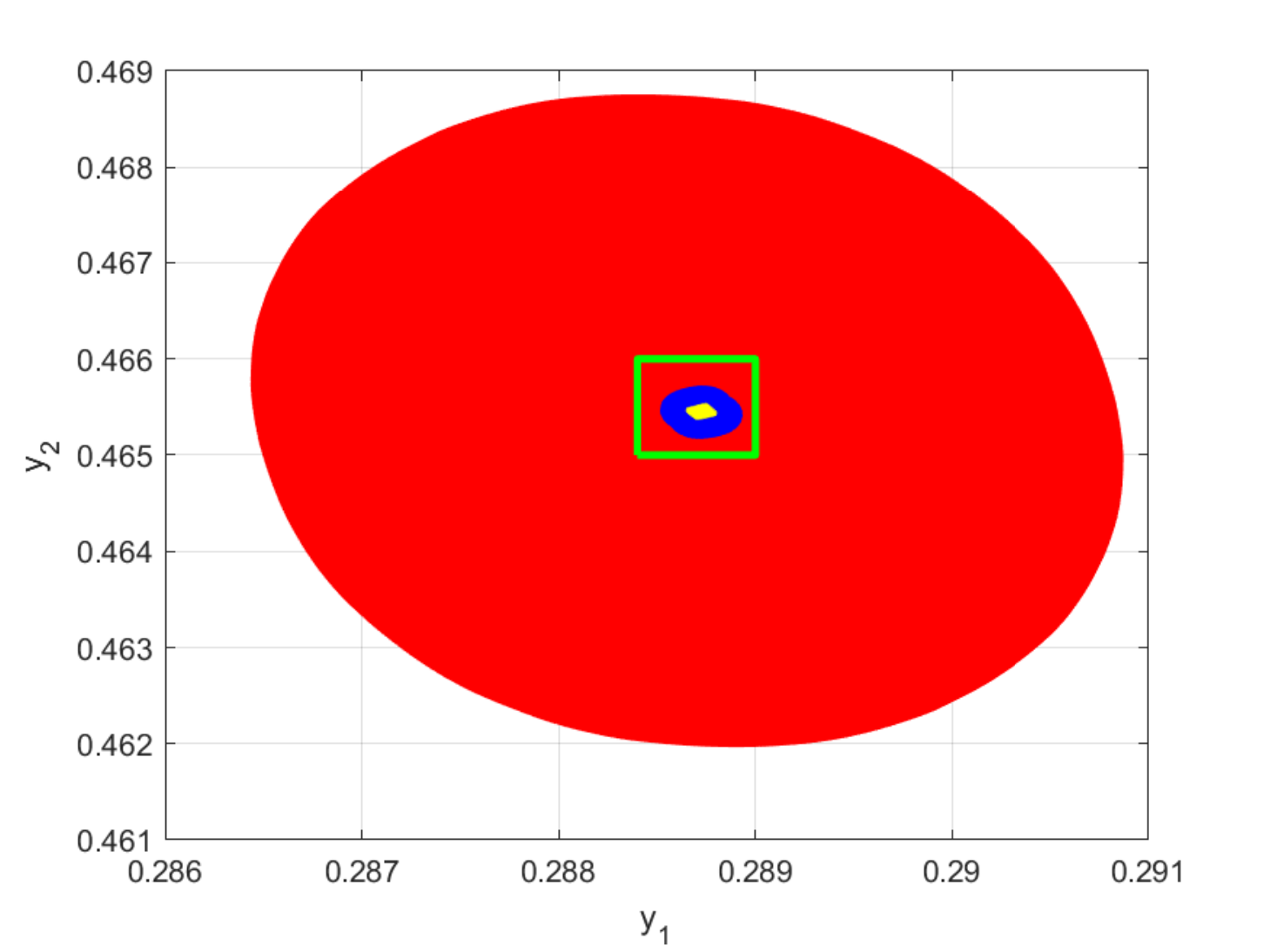}
\label{0.375-1}
\end{minipage}%
}%
\subfigure[$y_1-y_3$]{
\begin{minipage}[t]{0.3\linewidth}
\centering
\includegraphics[width=1.5in]{ 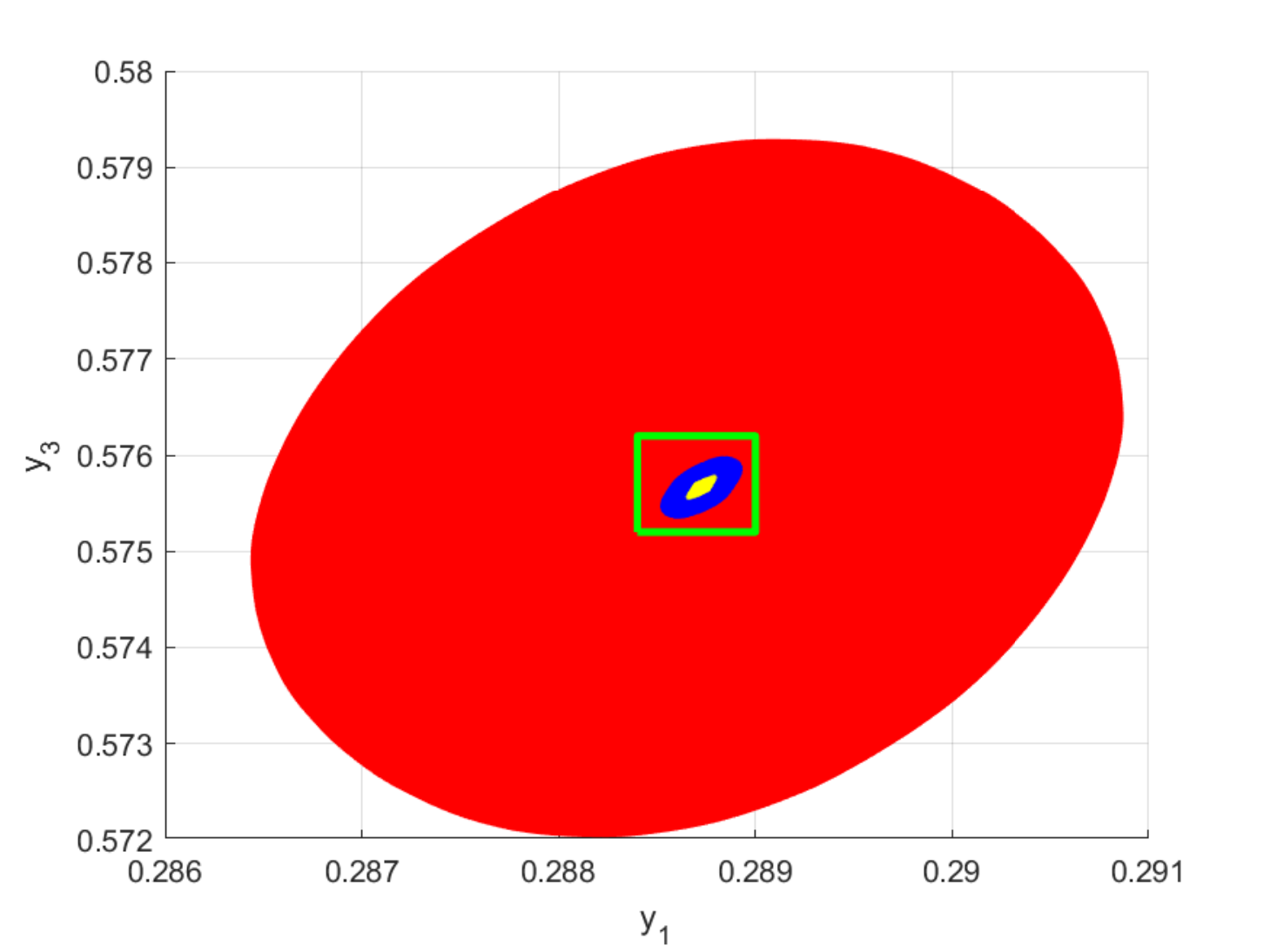}
\label{0.375-2}
\end{minipage}%
}%
\subfigure[$y_2-y_3$]{
\begin{minipage}[t]{0.3\linewidth}
\centering
\includegraphics[width=1.5in]{ 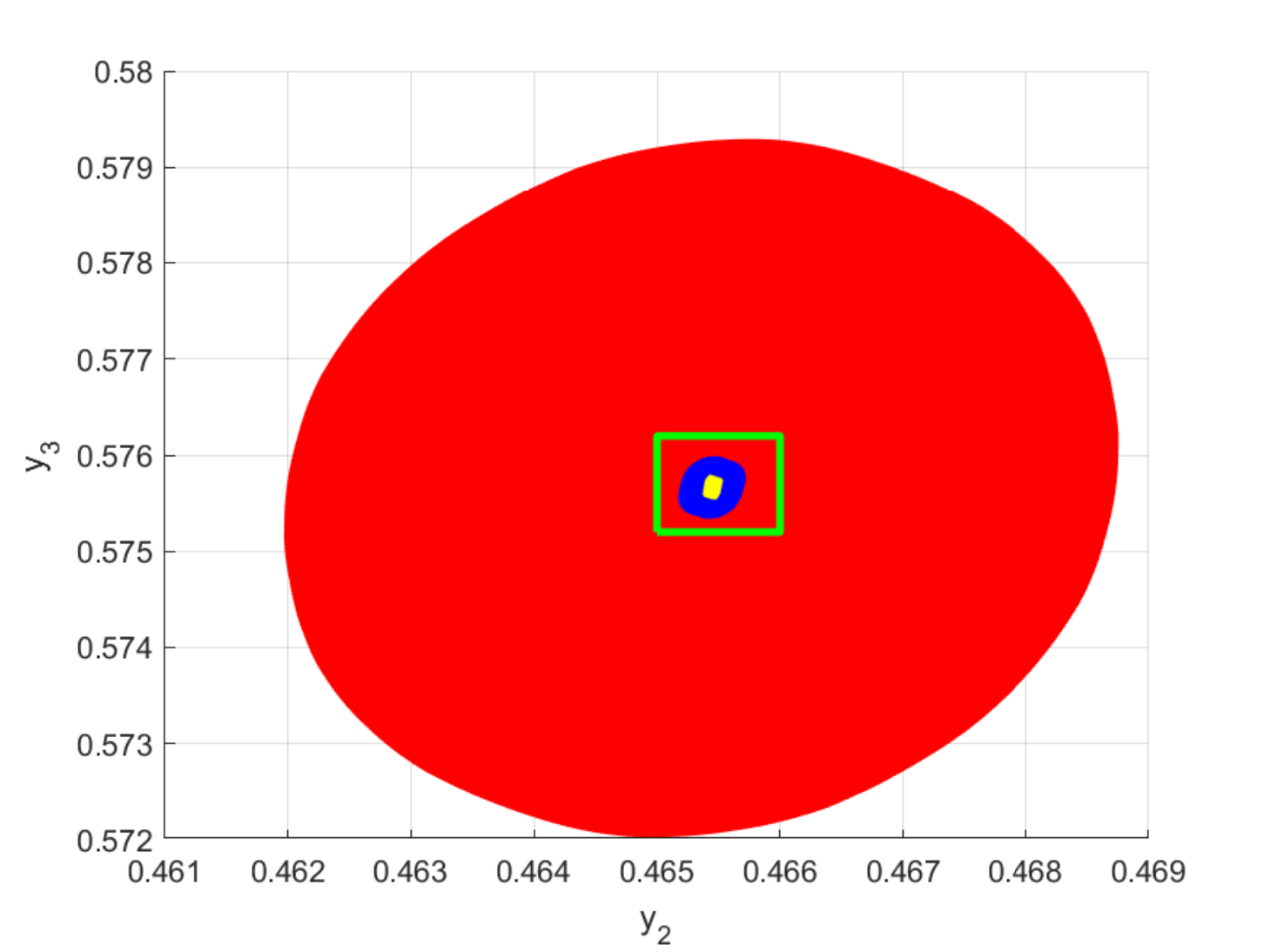}
\label{0.375-3}
\end{minipage}
}%%
\\
$\epsilon=0.375$, \textcolor{blue}{Safe}, \textcolor{red}{Unknown}
\centering
\caption{Safety verification on $\bm{N}_4$.
\textcolor{blue}{$\Omega(\partial \mathcal{X}_{in})$}; \textcolor{red}{$\Omega(\mathcal{X}_{in})$}; \textcolor{green}{$\partial \mathcal{X}_{s}$}}
\label{3-dim inn1}
\end{figure}

\oomit{\textcolor{red}{
\begin{remark} Though we demonstrate the performance of our set-boundary based verification method from neural ODEs and other invertible NNs separately, the proposed method also works on the combination NNs of these invertible ones, like GNODEs defined in \cite{manzanas2022reachability} (c.f. Definition 3.), as long as each component satisfies the invertibility requirement.
\end{remark}}
}

\subsection{Experiments on Non-invertible NNs}
When homeomorphisms cannot be assured with respect to given input regions, our method is also able to facilitate the extraction of subsets from the input region for safety verification, as done in Algorithm \ref{alg: iNNs1}. 
In this subsection,
we experiment on a non-invertible NN $\bm{N}_5$, which shares the same structure with $\bm{N}_3$. The input set $\mathcal{X}_{in}$ and safe set $\mathcal{X}_s$ are $[-0.5, 0.5]^2$ and $[0.06546,0.06555]\times[0.07828,0.07832]$, respectively. Then, based on the tool  DeepZ, we follow the computational procedure in Algorithm \ref{alg: iNNs1} for verifying the safety property. The computed output reachable sets and the verification result are shown in Fig. \ref{2-dim non-inn}. The subset $\mathcal{A}=[-0.45,0.3]^{2}$ rendering the NN homeomorphic is illustrated in  Fig. \ref{initial}, which is the orange region, and the subset $\overline{\mathcal{X}_{in}\setminus \mathcal{A}}$ is the blue region in Fig. \ref{initial}. It can be seen that the subset $\mathcal{A}$ extracted by set-boundary analysis for safety verification covers only 56.25\% of the initial input set. The output reachable set computed from the entire input set is also displayed in Fig.\ref{result}, which correspond to the red region. %Moreover, the reachable sets in different blue colors in Fig.\ref{result}  are computed from the corresponding color block in Fig.\ref{initial}. 
Moreover, the boundary of the safe region and the output reachable set estimated via the Monte-Carlo simulation method are shown in green and yellow in Fig.\ref{result}, respectively. It can be observed that our set-boundary reachability method facilitates the generation of a tighter output reachable set, which is included in the safe set $\mathcal{X}_s$. Thus, the safety property is ensured by our set-boundary reachability method. However, the entire set based method fails.  Moreover, the computation time of safe verification on $\bm{N}_5$ based on our set-boundary reachability method is 0.0459 seconds, while the verification time from the entire  set based method takes 0.0522 seconds with 4 equal subsets.

%, without considering the homeomorphism subset, the result of DeepZ becomes more compact and the verification conclusion turns from ``\textbf{Unknown}" to be ``\textbf{Safe}".

\begin{figure}[htbp]
\centering
\subfigure[$\mathcal{X}_{in}$: \textcolor{orange}{$\mathcal{A}$} $\cup $\textcolor{blue}{$\overline{\mathcal{X}_{in}\setminus  \mathcal{A}}$}]{
\begin{minipage}[t]{0.4\linewidth}
\centering
\includegraphics[width=1.5in]{ 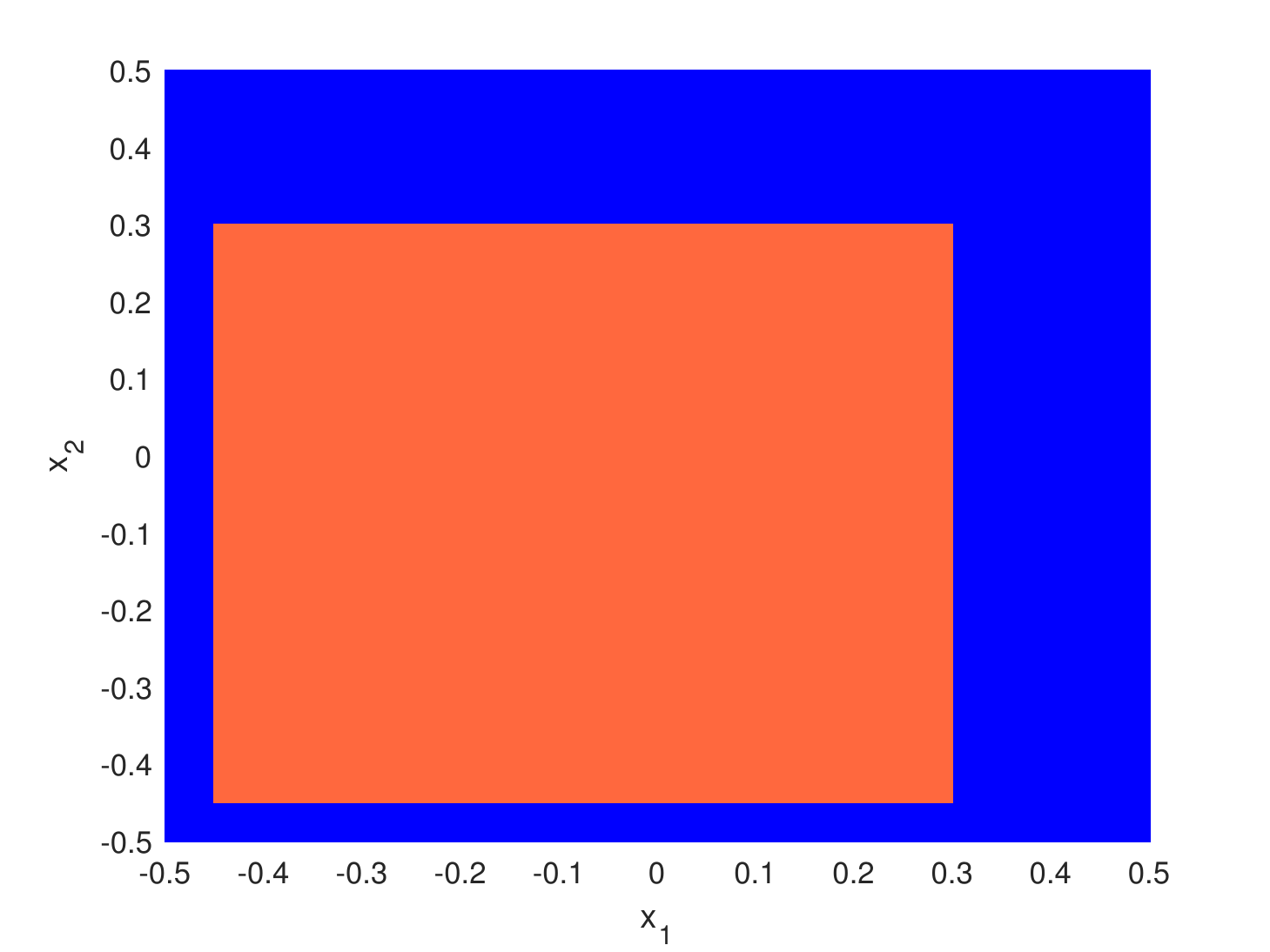}
%\caption{fig1}
\label{initial}
\end{minipage}%
}%
\centering
\subfigure[Verification: \textcolor{blue}{Safe}, \textcolor{red}{Unknown}.
\textcolor{blue}{$\Omega(\overline{\mathcal{X}_{in}\setminus  \mathcal{A}})$}; \textcolor{red}{$\Omega(\mathcal{X}_{in})$}; \textcolor{green}{$\partial \mathcal{X}_{s}$}]
{
\begin{minipage}[t]{0.6\linewidth}
\centering
\includegraphics[width=1.5in]{ 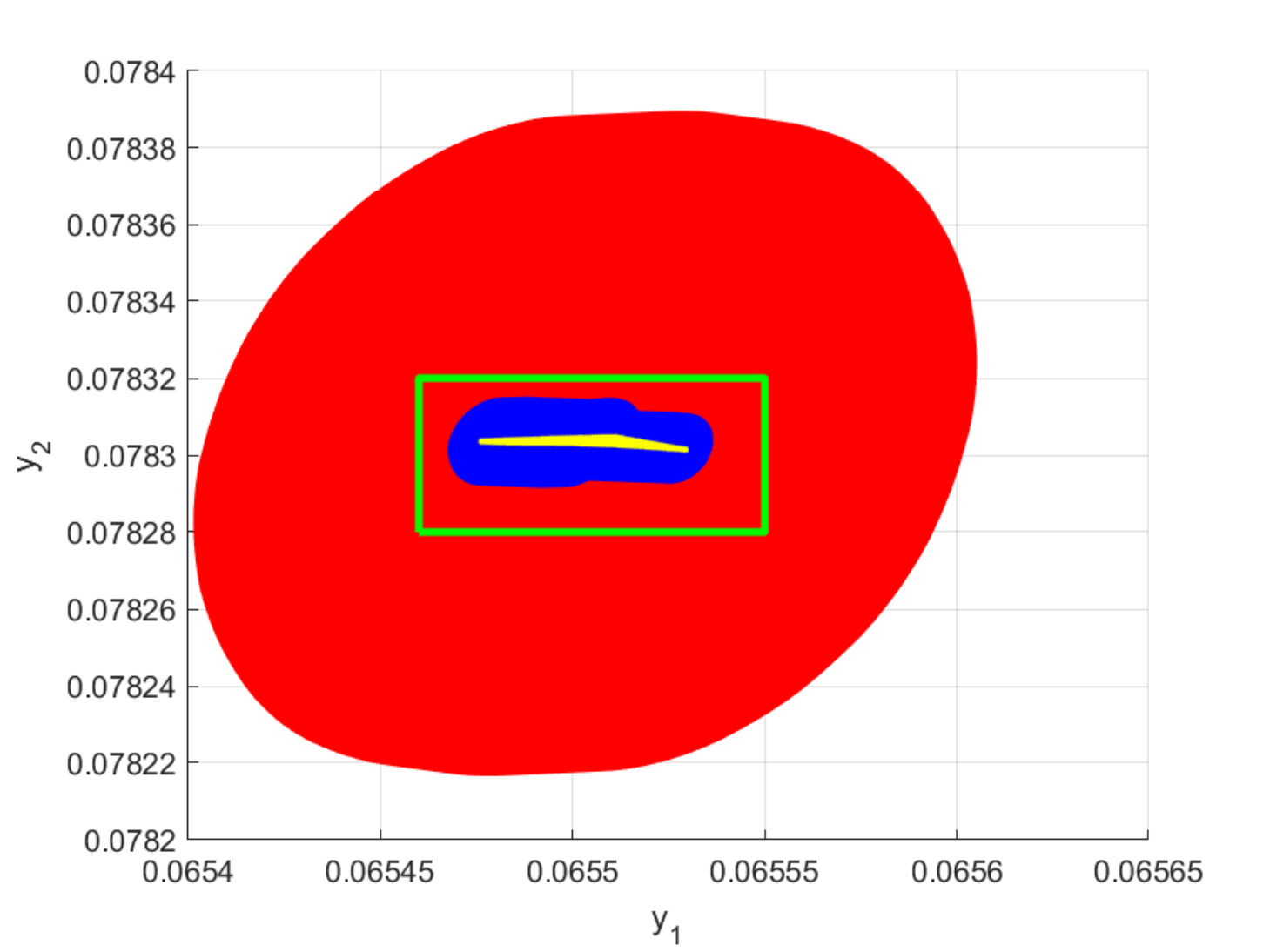}
%\caption{fig2}
\label{result}
\end{minipage}%
}%
\centering
\caption{Safety verification on $\bm{N}_5$}
\label{2-dim non-inn}
\end{figure}

\section{Conclusion}
In this paper we proposed a set-boundary reachability method to verify safety property of NNs. Different from existing works on developing computational techniques for output reachable sets estimation of NNs,  the set-boundary reachability method analyzed the reachability from the topology point of view. Based on homeomorphism property, this analysis took a careful inspection on what happens at boundaries of input sets, and uncovered that the homeomorphism property facilitates the reduction of computational burdens on safety verification of NNs. Several examples demonstrated the performance of the proposed method. 

There are a lot of works remaining to be done in order to render the proposed approach more practical. For instance, in this paper a homeomorphism is determined via the use of interval arithmetic to calculate the determinant of the Jacobian matrix. Such an interval estimation is coarse, which affects the determination of a homeomorphism and thus the extraction of the small subset for reachability computations. In the future we will develop more efficient and accurate methods for calculating Jacobian matrices. Besides, the homeomorphism property may be strict. Different from homeomorphisms, open maps, mapping open sets to open sets \cite{mendelson1990introduction}, can also ensure that the output reachable set's boundary corresponds to the input's boundary. Moreover, the open mapping condition is weaker than the one for a homeomorphism. Consequently, in future work we would exploit the open mapping property to facilitate reachability computations for safety verification.

%ensuring the exact computation from input boundary to output boundary, the boundary computation of open mapping seems to be redundant at some cases, for a boundary point may be mapped to an interior point with an open mapping. Whereas, the computation based on open mapping still is sound and the redundancy can be ignored compared with the computation on the entire set. Last but not least, the open mapping condition is much easier to hold on in general neural networks.

\bibliographystyle{splncs04}
\bibliography{reference.bib}

\begin{thebibliography}{10}
\providecommand{\url}[1]{\texttt{#1}}
\providecommand{\urlprefix}{URL }
\providecommand{\doi}[1]{https://doi.org/#1}

\bibitem{akintunde2019verification}
Akintunde, M.E., Kevorchian, A., Lomuscio, A., Pirovano, E.: Verification of
  rnn-based neural agent-environment systems. In: Proceedings of the AAAI
  Conference on Artificial Intelligence. vol.~33, pp. 6006--6013 (2019)

\bibitem{althoff2015introduction}
Althoff, M.: An introduction to cora 2015. In: Proc. of the workshop on applied
  verification for continuous and hybrid systems. pp. 120--151 (2015)

\bibitem{ardizzone2018analyzing}
Ardizzone, L., Kruse, J., Rother, C., K{\"o}the, U.: Analyzing inverse problems
  with invertible neural networks. In: International Conference on Learning
  Representations (2018)

\bibitem{behrmann2019invertible}
Behrmann, J., Grathwohl, W., Chen, R.T., Duvenaud, D., Jacobsen, J.H.:
  Invertible residual networks. In: International Conference on Machine
  Learning. pp. 573--582. PMLR (2019)

\bibitem{chen2018neural}
Chen, R.T., Rubanova, Y., Bettencourt, J., Duvenaud, D.K.: Neural ordinary
  differential equations. Advances in neural information processing systems
  \textbf{31} (2018)

\bibitem{cousot1977abstract}
Cousot, P., Cousot, R.: Abstract interpretation: a unified lattice model for
  static analysis of programs by construction or approximation of fixpoints.
  In: Proceedings of the 4th ACM SIGACT-SIGPLAN symposium on Principles of
  programming languages. pp. 238--252 (1977)

\bibitem{dahnert2021panoptic}
Dahnert, M., Hou, J., Nie{\ss}ner, M., Dai, A.: Panoptic 3d scene
  reconstruction from a single rgb image. Advances in Neural Information
  Processing Systems  \textbf{34} (2021)

\bibitem{dupont2019augmented}
Dupont, E., Doucet, A., Teh, Y.W.: Augmented neural odes. Advances in Neural
  Information Processing Systems  \textbf{32} (2019)

\bibitem{dutta2017output}
Dutta, S., Jha, S., Sanakaranarayanan, S., Tiwari, A.: Output range analysis
  for deep neural networks. arXiv preprint arXiv:1709.09130  (2017)

\bibitem{ehlers2017formal}
Ehlers, R.: Formal verification of piece-wise linear feed-forward neural
  networks. In: International Symposium on Automated Technology for
  Verification and Analysis. pp. 269--286. Springer (2017)

\bibitem{gehr2018ai2}
Gehr, T., Mirman, M., Drachsler-Cohen, D., Tsankov, P., Chaudhuri, S., Vechev,
  M.: Ai2: Safety and robustness certification of neural networks with abstract
  interpretation. In: 2018 IEEE symposium on security and privacy (SP). pp.
  3--18. IEEE (2018)

\bibitem{ghorbani2019interpretation}
Ghorbani, A., Abid, A., Zou, J.: Interpretation of neural networks is fragile.
  In: Proceedings of the AAAI conference on artificial intelligence. vol.~33,
  pp. 3681--3688 (2019)

\bibitem{gomez2017reversible}
Gomez, A.N., Ren, M., Urtasun, R., Grosse, R.B.: The reversible residual
  network: Backpropagation without storing activations. Advances in neural
  information processing systems  \textbf{30} (2017)

\bibitem{gruenbacher2020lagrangian}
Gruenbacher, S., Cyranka, J., Lechner, M., Islam, M.A., Smolka, S.A., Grosu,
  R.: Lagrangian reachtubes: The next generation. In: 2020 59th IEEE Conference
  on Decision and Control (CDC). pp. 1556--1563. IEEE (2020)

\bibitem{gruenbacher2021gotube}
Gruenbacher, S., Lechner, M., Hasani, R., Rus, D., Henzinger, T.A., Smolka, S.,
  Grosu, R.: Gotube: Scalable stochastic verification of continuous-depth
  models. arXiv preprint arXiv:2107.08467  (2021)

\bibitem{grunbacher2021verification}
Grunbacher, S., Hasani, R., Lechner, M., Cyranka, J., Smolka, S.A., Grosu, R.:
  On the verification of neural odes with stochastic guarantees. In:
  Proceedings of the AAAI Conference on Artificial Intelligence. vol.~35, pp.
  11525--11535 (2021)

\bibitem{hasani2020natural}
Hasani, R., Lechner, M., Amini, A., Rus, D., Grosu, R.: A natural lottery
  ticket winner: Reinforcement learning with ordinary neural circuits. In:
  International Conference on Machine Learning. pp. 4082--4093. PMLR (2020)

\bibitem{huang2019reachnn}
Huang, C., Fan, J., Li, W., Chen, X., Zhu, Q.: Reachnn: Reachability analysis
  of neural-network controlled systems. ACM Transactions on Embedded Computing
  Systems (TECS)  \textbf{18}(5s),  1--22 (2019)

\bibitem{huang2017safety}
Huang, X., Kwiatkowska, M., Wang, S., Wu, M.: Safety verification of deep
  neural networks. In: International conference on computer aided verification.
  pp. 3--29. Springer (2017)

\bibitem{ivanov2020verifying}
Ivanov, R., Carpenter, T.J., Weimer, J., Alur, R., Pappas, G.J., Lee, I.:
  Verifying the safety of autonomous systems with neural network controllers.
  ACM Transactions on Embedded Computing Systems (TECS)  \textbf{20}(1),  1--26
  (2020)

\bibitem{jacobsen2018revnet}
Jacobsen, J.H., Smeulders, A., Oyallon, E.: i-revnet: Deep invertible networks.
  arXiv preprint arXiv:1802.07088  (2018)

\bibitem{joshi1983introduction}
Joshi, K.D.: Introduction to general topology. New Age International (1983)

\bibitem{karch2021grounding}
Karch, T., Teodorescu, L., Hofmann, K., Moulin-Frier, C., Oudeyer, P.Y.:
  Grounding spatio-temporal language with transformers. arXiv preprint
  arXiv:2106.08858  (2021)

\bibitem{katz2017reluplex}
Katz, G., Barrett, C., Dill, D.L., Julian, K., Kochenderfer, M.J.: Reluplex: An
  efficient smt solver for verifying deep neural networks. In: International
  conference on computer aided verification. pp. 97--117. Springer (2017)

\bibitem{krantz2002implicit}
Krantz, S.G., Parks, H.R.: The implicit function theorem: history, theory, and
  applications. Springer Science \& Business Media (2002)

\bibitem{lechner2020neural}
Lechner, M., Hasani, R., Amini, A., Henzinger, T.A., Rus, D., Grosu, R.: Neural
  circuit policies enabling auditable autonomy. Nature Machine Intelligence
  \textbf{2}(10),  642--652 (2020)

\bibitem{liu2021algorithms}
Liu, C., Arnon, T., Lazarus, C., Strong, C., Barrett, C., Kochenderfer, M.J.,
  et~al.: Algorithms for verifying deep neural networks. Foundations and
  Trends{\textregistered} in Optimization  \textbf{4}(3-4),  244--404 (2021)

\bibitem{liuww2020article}
Liu, W., Song, F., Zhang, T., Wang, J.: Verifying relu neural networks from a
  model checking perspective. Journal of Computer Science and Technology
  \textbf{35},  1365--1381 (11 2020). \doi{10.1007/s11390-020-0546-7}

\bibitem{lomuscio2017approach}
Lomuscio, A., Maganti, L.: An approach to reachability analysis for
  feed-forward relu neural networks. arXiv preprint arXiv:1706.07351  (2017)

\bibitem{lopez2022reachability}
Lopez, D.M., Musau, P., Hamilton, N., Johnson, T.T.: Reachability analysis of a
  general class of neural ordinary differential equations. arXiv preprint
  arXiv:2207.06531  (2022)

\bibitem{manzanas2022reachability}
Manzanas~Lopez, D., Musau, P., Hamilton, N., Johnson, T.T.: Reachability
  analysis of a general class of neural ordinary differential equations. arXiv
  e-prints pp. arXiv--2207 (2022)

\bibitem{massey2019basic}
Massey, W.S.: A basic course in algebraic topology, vol.~127. Springer (2019)

\bibitem{mendelson1990introduction}
Mendelson, B.: Introduction to topology. Courier Corporation (1990)

\bibitem{naitzat2020topology}
Naitzat, G., Zhitnikov, A., Lim, L.H.: Topology of deep neural networks. J.
  Mach. Learn. Res.  \textbf{21}(184),  1--40 (2020)

\bibitem{pulina2010abstraction}
Pulina, L., Tacchella, A.: An abstraction-refinement approach to verification
  of artificial neural networks. In: International Conference on Computer Aided
  Verification. pp. 243--257. Springer (2010)

\bibitem{singh2018fast}
Singh, G., Gehr, T., Mirman, M., P{\"u}schel, M., Vechev, M.: Fast and
  effective robustness certification. Advances in neural information processing
  systems  \textbf{31} (2018)

\bibitem{singh2019abstract}
Singh, G., Gehr, T., P{\"u}schel, M., Vechev, M.: An abstract domain for
  certifying neural networks. Proceedings of the ACM on Programming Languages
  \textbf{3}(POPL),  1--30 (2019)

\bibitem{tian2021image}
Tian, Y., Yang, W., Wang, J.: Image fusion using a multi-level image
  decomposition and fusion method. Applied Optics  \textbf{60}(24),  7466--7479
  (2021)

\bibitem{tran2019star}
Tran, H.D., Manzanas~Lopez, D., Musau, P., Yang, X., Nguyen, L.V., Xiang, W.,
  Johnson, T.T.: Star-based reachability analysis of deep neural networks. In:
  International symposium on formal methods. pp. 670--686. Springer (2019)

\bibitem{tran2019parallelizable}
Tran, H.D., Musau, P., Lopez, D.M., Yang, X., Nguyen, L.V., Xiang, W., Johnson,
  T.T.: Parallelizable reachability analysis algorithms for feed-forward neural
  networks. In: 2019 IEEE/ACM 7th International Conference on Formal Methods in
  Software Engineering (FormaliSE). pp. 51--60. IEEE (2019)

\bibitem{wang2018efficient}
Wang, S., Pei, K., Whitehouse, J., Yang, J., Jana, S.: Efficient formal safety
  analysis of neural networks. Advances in Neural Information Processing
  Systems  \textbf{31} (2018)

\bibitem{xiang2018reachability}
Xiang, W., Johnson, T.T.: Reachability analysis and safety verification for
  neural network control systems. arXiv preprint arXiv:1805.09944  (2018)

\bibitem{xiang2017reachable}
Xiang, W., Tran, H.D., Johnson, T.T.: Reachable set computation and safety
  verification for neural networks with relu activations. arXiv preprint
  arXiv:1712.08163  (2017)

\bibitem{xiang2018output}
Xiang, W., Tran, H.D., Johnson, T.T.: Output reachable set estimation and
  verification for multilayer neural networks. IEEE transactions on neural
  networks and learning systems  \textbf{29}(11),  5777--5783 (2018)

\bibitem{xue2016reach}
Xue, B., Easwaran, A., Cho, N.J., Fr{\"a}nzle, M.: Reach-avoid verification for
  nonlinear systems based on boundary analysis. IEEE Transactions on Automatic
  Control  \textbf{62}(7),  3518--3523 (2016)

\bibitem{xue2016under}
Xue, B., She, Z., Easwaran, A.: Under-approximating backward reachable sets by
  polytopes. In: International Conference on Computer Aided Verification. pp.
  457--476. Springer (2016)

\bibitem{xue2020over}
Xue, B., Wang, Q., Feng, S., Zhan, N.: Over-and underapproximating reach sets
  for perturbed delay differential equations. IEEE Transactions on Automatic
  Control  \textbf{66}(1),  283--290 (2020)

\bibitem{yang2021improving}
Yang, P., Li, R., Li, J., Huang, C.C., Wang, J., Sun, J., Xue, B., Zhang, L.:
  Improving neural network verification through spurious region guided
  refinement. In: International Conference on Tools and Algorithms for the
  Construction and Analysis of Systems. pp. 389--408. Springer (2021)

\bibitem{yuan2021bartscore}
Yuan, W., Neubig, G., Liu, P.: Bartscore: Evaluating generated text as text
  generation. arXiv preprint arXiv:2106.11520  (2021)

\end{thebibliography}
\end{document}